\documentclass[runningheads]{llncs}
\usepackage{amsmath}%
\usepackage[T1]{fontenc}
\usepackage{graphicx}
\usepackage{paralist}
\usepackage{tikz}
\usetikzlibrary{shapes, automata, fit}
\usepackage{todonotes}
\usepackage{xcolor, colortbl}
\usepackage{wrapfig}
\usepackage{comment}
\usepackage[subtle]{savetrees}
\usepackage[noend,ruled,linesnumbered]{algorithm2e}

\SetCommentSty{mycommfont}

\usepackage{amsmath,amssymb,amsfonts}
\usepackage{bbm}
\newcommand\commentOut[1]{}
\renewcommand\vec[1]{\mathbf{#1}}
\newcommand\scr[1]{\ensuremath\mathcal{#1}}
\newcommand\tupleof[1]{\ensuremath\langle #1 \rangle}
\newcommand\reals{\mathbb{R}}
\newcommand\D{\scr{D}}
\newcommand\setof[1]{\ensuremath\left\{ #1 \right\}}

\newcommand\vb{\vec{b}}
\newcommand\vx{\vec{x}}
\newcommand\vy{\vec{y}}
\newcommand\cond{\mathsf{condition}}

\newcommand\player[1]{{\cal P}_{#1}}
\newcommand\Ac[1]{A^{(#1)}}
\newcommand\Beliefs{\mathcal{B}}

\definecolor{Gray}{gray}{0.85}
\definecolor{LightCyan}{rgb}{0.88,1,1}
\newcolumntype{b}{>{\columncolor{LightCyan}}c}
\newtheorem{assumption}{Assumption}

\newcommand{\set}[1]{\left\{ #1 \right\}}

\usepackage{inconsolata}
\usepackage{cancel}
\newcommand\oneMat{\mathbbm{1}}

\newif\ifextendedversion\extendedversiontrue

\title{Anticipating Oblivious Opponents in Stochastic Games}
\author{Shadi Tasdighi Kalat,  
Sriram Sankaranarayanan%\orcidID{0000-0001-7315-4340} 
\and
Ashutosh Trivedi%\orcidID{0000-0001-9346-0126
}
\institute{
University of Colorado Boulder, USA \\ 
Email: \texttt{first.lastname @colorado.edu}\\}
\authorrunning{Kalat et al.}
\begin{document}
\maketitle              % typeset the header of the contribution

\begin{abstract}
We present an approach for systematically anticipating the actions and policies employed by  \emph{oblivious} environments in concurrent stochastic games, while maximizing a reward function.
% We present an approach for systematically anticipating the actions and policies employed by  \emph{oblivious} environments in concurrent stochastic games, while optimizing the controller's moves to maximize a reward function defined over the plays. 
% Two-player stochastic games offer an elegant and robust solution concept for controller synthesis in the presence of an adversarial environment. 
% However, most existing solutions assume that the opposing player is either adversarial (zero-sum) or cooperative.  Our approach focuses on studying an \emph{oblivious} player who  plays according to  a fixed memoryless strategy/policy drawn from a finite set, irrespective of the controller's actions.
% The controller, in turn, maintains a belief about the potential strategies of the environment and updates this belief by observing the actions of the environment. 
% Although this setting can be modeled as a partially observable Markov decision process (POMDP), traditional POMDP solution techniques do not fully leverage the underlying  game structure. 
Our main contribution lies in the synthesis of a finite \emph{information state machine} whose alphabet ranges over the actions of the environment. Each state of the automaton is mapped to a belief state about the policy used by the environment. We introduce a notion of  consistency that guarantees that the belief states tracked by our automaton stays within a fixed distance of the precise belief state obtained by knowledge of the full history. We provide methods for checking consistency of an automaton and a synthesis approach which upon successful termination yields such a machine. We show how the information state machine  yields an MDP that serves as the starting point for computing optimal policies for maximizing a reward function defined over plays. 
We  present an experimental evaluation over benchmark examples including human activity data for tasks such as cataract surgery and furniture assembly, wherein our approach successfully anticipates  the policies and actions of the environment in order to maximize the reward.
\end{abstract}

%
%
% Anticipation ...ss
\section{Introduction}
\emph{Concurrent stochastic games}~\cite{filar2012competitive,de2000control,de2001control,chatterjee2012survey,kochenderfer2015decision,shapley1953stochastic} offer a natural abstraction for modeling conservative decision-making in the presence of multiple agents in a shared and uncertain environment. 
In this scenario, the objective of the \emph{Ego} agent---player $\player1$---is to maximize their desired outcome irrespective of the decisions taken by other agents, represented here as a single agent that we term player $\player2$~\cite{bewley1978stochastic}.  In a
\emph{zero-sum game}, 
the objective of player $\player1$ is deemed to be in direct conflict with player $\player2$. The opposite scenario assumes \emph{cooperation}, wherein $\player2$'s actions are aimed to maximize the reward for $\player1$. In this paper, we study another ``extreme'', wherein $\player2$ is assumed to be 
\emph{oblivious}. Their actions are chosen from a predefined set of policies or objectives that are not affected by the actions of $\player1$.  
We will show that in such a setting, player $\player1$ needs to \emph{anticipate} $\player2$'s moves to maximize their own reward.

Consider a game of Rock-paper-scissors (RPS) against an oblivious adversary. 
Recall that at each turn, players $\player1$ and $\player2$ simultaneously reveal their choice with a show of hands, and both players receive values (see, Figure~\ref{fig:rps-rewards-and-policies}) based on straightforward circular-dominance rules (rock defeats scissors, scissors defeats paper, paper defeats rock). 
%The adversarial case can be seen as a zero-sum matrix game, leading to the well-known policy of perfectly mixing between the three choices to minimize expected losses.
The repeated, oblivious RPS can be modeled as a single state concurrent stochastic game, where the goal of player $\player1$ is to maximize the sequence of rewards over a given, potentially infinite, horizon.
Considering the conventional interpretation of an adversarial opponent, the expected value of the game remains at $0$.  
%In an oblivious setting, $\player2$ shuffles among a finite set of  policies, such as consistently selecting rock or mimicking player $\player1$'s previous move. Can player $\player1$ leverage the patterns of $\player2$'s moves to their own advantage?

The oblivious RPS ``game'' is illustrated in Figure~\ref{fig:rps-rewards-and-policies}, where the set of policies ($\pi_1, \pi_2, \pi_3$, and $\pi_4$) used by player $\player2$ is presented in the table to the right.
In the proposed scenario, we assume the following:
\begin{inparaenum}[(a)]
\item player $\player1$ observes the \emph{past} actions of player $\player2$ but the \emph{current} action of one player is not observable by the other;
\item player $\player2$ is restricted to playing one of the policies $\{\pi_1, \pi_2, \pi_3, \pi_4\}$ but this choice is \emph{not observable} by $\player1$; and 
\item \emph{policy change:} at each step, player $\player2$ may shift from the current policy to a new one. This shift is modeled by a Markov chain wherein each state of the chain is labeled by a policy. 
\end{inparaenum}
From player $\player1$'s perspective, although the policies of player $\player2$ are known, they are unobservable. Consequently, the problem can be framed as a partially observable MDP (POMDP). 
This POMDP is the result of merging the original arena with 
player $\player2$'s policy set.
Framing this as a POMDP permits the use of standard POMDP solution approaches~\cite{cassandra1998survey}. However, ``exact'' POMDP planning is undecidable~\cite{madani1999undecidability}. Furthermore, translating from oblivious games to POMDPs obscures the specialized structure of the problem.

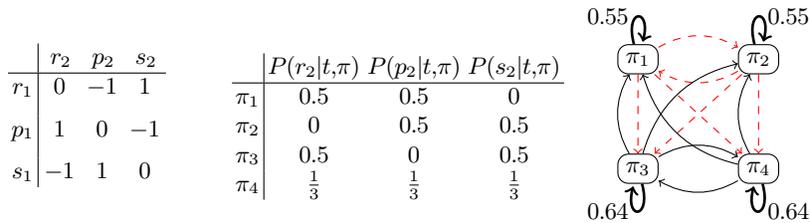
\begin{figure}[t]
\begin{tabular}{ccc}
\begin{minipage}{0.3\textwidth}
\[ \begin{array}{c| ccc}
     & \;\;\; r_2 \;\;\; & \;\;\; p_2 \;\;\; &  \;\;\; s_2 \;\;\; \\ 
     \hline 
r_1 &   0  & -1 &  1 \\[5pt]
p_1 &  1  & 0 & -1 \\[5pt]
s_1 & -1 & 1 & 0  \\
\end{array} \] 
\end{minipage} & 
\begin{minipage}{0.3\textwidth} 
\[ \begin{array}{c | c c c }
& P(r_2| t, \pi) & P(p_2 | t, \pi) & P(s_2 | t,\pi) \\ 
\hline 
\pi_1 & 0.5 & 0.5 & 0 \\ 
\pi_2 & 0 & 0.5 & 0.5 \\ 
\pi_3 & 0.5 & 0 & 0.5 \\ 
\pi_4 & \frac{1}{3} & \frac{1}{3} & \frac{1}{3}\\ 
\end{array} \] \end{minipage} & 
\begin{minipage}{0.4\textwidth} 
\begin{center}
\begin{tikzpicture}
\matrix[every node/.style={rectangle, rounded corners, draw=black}, row sep=30pt, column sep=30pt]{
\node(n1){$\pi_1$}; & \node(n2){$\pi_2$}; \\ 
\node(n3){$\pi_3$}; & \node(n4){$\pi_4$}; \\ 
};
\path[->, line width=1pt] (n1) edge[loop above] node[left] {0.55} (n1)
(n2) edge[loop above] node[right] {0.55} (n2)
(n3) edge[loop below] node[left] {0.64} (n3)
(n4) edge[loop below] node[right] {0.64} (n4);
\path[->, thin, dashed, red] (n1) edge[bend left] (n2)
(n1) edge (n3)
(n1) edge (n4)
(n2) edge[bend left] (n1)
(n2) edge (n3)
(n2) edge (n4);
\path[->, thin] (n3) edge[bend left] (n1)
(n3) edge[bend left] (n2)
(n4) edge[bend left](n1)
(n4) edge[bend left] (n2)
(n3) edge[bend left] (n4)
(n4) edge[bend left] (n3);
\end{tikzpicture}
\end{center}
\end{minipage}\\

\end{tabular}
\caption{Rock-paper-scissors (RPS) game arena. Here actions $r_i$, $p_i$, and $s_i$ correspond to the choices of ``rock'', ``paper'' and ``scissors'' by player $\player{i}$; (\textbf{left}) Reward table; (\textbf{mid}) player $\player2$ policies; and (\textbf{right}) Markov chain modeling policy change for $\player{2}$. The dashed red edges have probability $0.15$ whereas the solid edges have probability $0.12$.}
\vspace{-1.5em}
\label{fig:rps-rewards-and-policies}
\end{figure}

\paragraph{Action/Tool Anticipation in Human-Robot Cooperative Tasks:} In  scenarios involving humans working with autonomous agents, the ability  to ``guess'' the intent of the human can be critical in ensuring the success of the overall task.
Consider a scenario where $\player2$ is engaged in a complex task involving a sequence of steps such as assembling a piece of furniture~\cite{Ben-Shabat+Others/2021/WACV} or performing a cataract surgery~\cite{Hassan+Others/2019/CATARACTS}. 
\begin{figure}[t]
\begin{center}
\begin{tikzpicture}[stn/.style={circle, draw=black, inner sep=1pt}]
\begin{scope}[scale=0.85, xshift=-1cm]
\node[stn, fill=green!20](t0) at (1,4.2) {\scriptsize $t_0$};
\node[stn](t1) at (0.3,3.5) {\scriptsize $t_1$};
\node[stn](t2) at (1.7,3.5) {\scriptsize $t_2$};
\node[stn](t3) at (0.3, 2.7) {\scriptsize $t_3$};
\node[stn](t4) at (1.7, 2.7) {\scriptsize $t_4$};
\node[stn](t5) at (0.3, 1.9) {\scriptsize $t_5$};
\node[stn](t6) at (1.7, 1.9) {\scriptsize $t_6$};
\node[stn](t7) at (1, 1.2){\scriptsize $t_7$};
\node[stn, fill=red!20](t8) at (1, 0.4){\scriptsize $t_8$};
\path[->, line width=1.5pt] (t0) edge node[above]{\scriptsize $a_1~~~$} (t1)
(t0) edge node[above]{\scriptsize $~~~a_4$} (t2)
(t1) edge node[left]{\scriptsize $a_2$} (t3)
(t3) edge node[left]{\scriptsize $a_4$} (t5)
(t1) edge node[below]{\scriptsize $a_4$} (t4)
(t5) edge node[left]{\scriptsize $~~a_3$} (t7)
(t7) edge node[right]{\scriptsize $a_2$} (t8)
(t2) edge node[right]{\scriptsize $a_4$} (t4)
(t4) edge node[right]{\scriptsize $a_2$} (t6)
(t6) edge node[below]{\scriptsize $~~~a_5$} (t7)
(t6) edge node[below](a3){\scriptsize $a_3$} (t3);
\end{scope}
%\node[fit=(t0)(t1)(t2)(t3)(t4)(t5)(t6)(t7)(t8)(a3), rectangle, rounded corners, dashed, draw=black, inner sep=10pt](nn){};
%\node at (nn.north)[rectangle, rounded corners, fill=gray!20]{Task $T$};
\begin{scope}[xshift=5cm, yshift=1.8cm]
\node at (0, 0) {
\begin{minipage}{0.6\textwidth}
\footnotesize 
\begin{tabular}{ll}
\hline 
$\pi_1: $ & $ t_0 \mapsto a_1, t_1 \mapsto \{a_2: 0.9, a_4: 0.1\}$ \\ 
& $t_6 \mapsto \{ a_3: 0.5, a_5: 0.5 \}, \cdots$ \\[5pt]
$\pi_2:$ & $t_0 \mapsto a_4, t_6 \mapsto \{ a_3: 0.5, a_5: 0.5\}, \cdots $ \\[5pt]
$\pi_3:$ & $t_0 \mapsto \{a_1: 0.5, a_4:0.5\}, t_6 \mapsto \{a_3: 0.5, a_5: 0.5 \}$, \\ 
& $t_1 \mapsto a_2,\ \ldots $ \\[5pt]
$\pi_4:$ & $t_0 \mapsto a_1, t_1 \mapsto \{ a_2: 0.5, a_4: 0.5\}$,\\
& $t_6 \mapsto \{a_3: 0.5, a_5: 0.5\}, \ldots $\\
\hline 
\end{tabular}
\end{minipage}
};
\end{scope}

\begin{comment}
    \begin{scope}[xshift=3.5cm]
    %\draw[help lines, step=0.2] (0,0) grid (3,4);
    \node[stn, fill=green!20](r0) at (1,4.2) {$r_0$};
    \node[stn](r1) at (0.3,3.5) {$r_1$};
    \node[stn](r2) at (1.7,3.5) {$r_2$};
    \node[stn](r3) at (0.3, 2.7) {$r_3$};
    \node[stn](r4) at (1.7, 2.4) {$r_4$};
    \node[stn](r5) at (0.3, 1.9) {$r_5$};
    \node[stn](r6) at (1.7, 1.4) {$r_6$};
    \node[stn](r7) at (0.3, 1.2){$r_7$};
    \node[stn, fill=red!20](r8) at (1, 0.4){$r_8$};
    \path[->, line width=1.5pt] (r0) edge node[above]{$a_1~~~$} (r1)
    (r0) edge node[above]{$~~~a_1$} (r2)
    (r1) edge node[right]{$a_2$} (r3)
    (r3) edge node[right]{$a_4$} (r5)
    (r5) edge node[right]{$a_3$} (r7)
    (r7) edge node[right]{$a_2$} (r8)
    (r2) edge node[right](z3){$a_4$} (r4)
    (r4) edge node[right]{$a_2$} (r6)
    (r6) edge node[right]{$a_1$} (r8)
    (r4) edge[bend left] node[right]{$a_2$} (r0);
    
    \node[fit=(r0)(r1)(r2)(r3)(r4)(r5)(r6)(r7)(r8)(z3), rectangle, rounded corners, dashed, draw=black, thin, inner sep=10pt](nn2){};
    \node at (nn2.north)[rectangle, rounded corners, fill=gray!20]{Task $T_2$};
    \end{scope}
\end{comment}
\end{tikzpicture}

\end{center}
%\vspace*{-0.4cm}
\caption{States of a furniture assembly task and policies for task completion. }\label{fig:furniture-assembly-task-graph}
\vspace*{-0.4cm}
\end{figure}
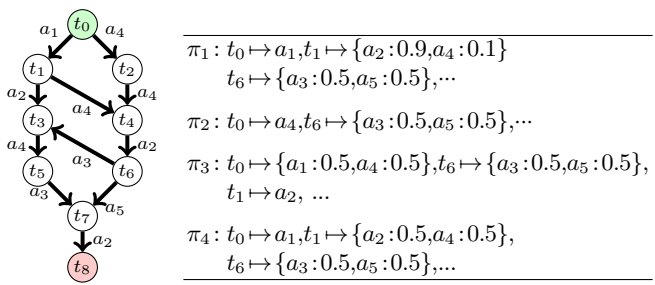
 
 The task execution is captured by a task graph whose nodes model different states encountered during task execution and  edges are labeled with the tool/action that is needed to move from one stage to another. Fig.~\ref{fig:furniture-assembly-task-graph} shows such a graph: the states $S_1 {=} \{ t_0, \ldots, t_8\}$  represent assembly stages for the corresponding component, while $A = \{ a_1, \ldots, a_5 \}$ represents the actions taken. Multiple edges from the same node represent possible choices that can be made by $\player2$.  
 The policies of $\player2$  dictate the choices made by $\player2$ for each non-terminal state. For some of the states with just one outgoing edge, there is just one choice to be made. However, for states with multiple outgoing edges, the policy dictates the probability distribution of the choice.  The policies allow us to model ``correlations'' in $\player2$'s action:  For instance, policy $\pi_1$ models the rule: $\player2$ chooses tool $a_1$ at state $t_0$ and they will choose  $a_2$ at state $t_1$ with $90\%$ probability.  The goal of $\player1$ is to accurately anticipate $\player2$'s choice of the next tool in order to perform a cooperative action (eg., pre-fetch the tool to help $\player2$ or automatically take steps to protect $\player2$ against a known hazard). We model this using the reward structure: if $\player1$ correctly predicts the next action of $\player2$ they obtain a positive reward. However, failure to do so incurs a negative reward. By assuming a set of policies for $\player2$, our approach moves the prediction problem from one of simply predicting action sequences to first predicting the policy (or the internal logic behind $\player2$'s actions) and then  predicting the action given the policy.  Section~\ref{sec:experimental-evaluation} demonstrates how we can use actual observation data from real-life cataract surgeries and furniture assembly tasks to not just learn the task graph  model but also infer policies. In doing so, our approach can produce policies for $\player1$ that predict the next action with upto $40\%$ accuracy even when there are more than $30$ tools/actions to choose from at each step.

\vspace{0.5em}\noindent\textbf{Contributions.}  In this paper, we introduce the framework of anticipation games  and formalize these games as POMDPs (Section~\ref{sec:preliminaries}). 
\begin{inparaenum}
    \item \emph{ Consistent Information State Machines:} We define the notion of a finite information state machine (ISM) over an alphabet consisting of states and $\player2$ actions (Section~\ref{sec:checking-consistency}). We introduce the concept of $\lambda-$consistency that is similar to an approximation bisimuation relation and show how to check if a given state machine is $\lambda$ consistent using  linear arithmetic SMT solvers. Next, we provide a semi-algorithm that upon success can synthesize such a machine (Section~\ref{sec:consistent-ism-synthesis}).  We provide simple conditions that guarantee the successful termination of our algorithm with a finite state consistent ISM (Section~\ref{sec:completeness-and-robustness}). 
    \item \emph{Policy Synthesis for $\player1$:} Next we show that a composition of a finite state ISM with the game yields a MDP that forms the basis of finding a policy for $\player1$ (Section~\ref{sec:policy-synthesis}). We bound the distances between the transition probabilities and reward functions of the infinite state belief MDP and the finite state approximation. By leveraging a recent  result by Subramanian et al.~\cite{subramanian2022approximate}, we bound the gap between the optimal belief MDP value function and that of our finite approximation.
    \item \emph{Robustness:} 
    % Although our 
    In Section~\ref{sec:completeness-and-robustness}, we establish bounds on the performance degradation, if $\player2$ deviates from the assumptions.
    % Our approach makes assumptions on how $\player2$ behaves, we establish bounds on the performance degradation, even if $\player2$ deviates from these assumptions in certain ways (Section~\ref{sec:completeness-and-robustness}). 
   \item  \emph{Empirical Evaluation:} Finally, we present an empirical evaluation of our work against some challenging benchmarks (Section~\ref{sec:experimental-evaluation}). We show that our approach can clearly anticipate the policies and actions of the other player to maximize the overall reward. In particular, we use two datasets -- an IKEA furniture assembly dataset consisting of sequence of actions taken by human assemblers for different furniture models~\cite{Ben-Shabat+Others/2021/WACV} and a sequence of tools used in 25 different cataract surgeries~\cite{Hassan+Others/2019/CATARACTS}. We use an automata learning tool flexfringe~\cite{Verwer+Others/2017/flexfringe} to learn the task model and a simple edge set based clustering to learn policies. We demonstrate how our approach computes policies for $\player1$ that maximize the ability to predict the next tool choice of $\player2$.
\end{inparaenum}

 \commentOut{
The key idea of the paper is to use sequences of observations of $\player2$'s states and actions to infer a \emph{belief state}  over the player's policies. 
This information is stored in the form of a \emph{belief state} given as a probability distribution over $\player2$'s policies. 
While the set of belief states can generally be uncountably infinite, our approach approximates this information by representing the experience as a Moore machine, which we refer to as the \emph{information state machine}. 
To maintain a bound on the deviation between the actual belief state and the one governed by the information state machine—defined via the \emph{total-variation distance}—we introduce the concept of (approximate) consistency of transitions within the information state machines. 
The consistency of transitions can be reduced to checking a formula in linear arithmetic.
Moreover, the existence of consistent information state machine allows us to extend this guarantees to any arbitrary sequence of transitions within the machine. }

 \vspace{0.5em}\noindent\textbf{Related Work.} 
Partially observable stochastic games (POSGs) are a subset of stochastic games where agents have partial information about the state of the environment. 
Within this paradigm, agents are allowed to have conflicting, or similar objectives, reward structures, and strategies~\cite{cassandra1998survey,boutilier1999decision,bernstein2002complexity}. 
Solution techniques developed for POSGs are build upon approaches to solve POMDPs such as value iteration and policy iteration~\cite{bellman1957markovian}. 
% These approaches are classical dynamic programming methods~\cite{hansen2004dynamic} that suffer from scalability issues in large state and observation spaces. 
Solving finite-horizon POMDP is PSPACE-complete~\cite{papadimitriou1987complexity}, and solving infinite-horizon POMDPs have been shown to be undecidable~\cite{madani2003undecidability}. 
A variety of approximate solution techniques have been introduced for general POMDPs including Point-Based Value iteration~\cite{pineau2003point,shani2007forward,spaan2005perseus,kim2011point,roy2002exponential}, grid-based belief MDP approximations~\cite{hauskrecht2000value}, semi-MDP approximations~\cite{theocharous2003approximate,strauch1966negative} and compressing belief states using features~\cite{horak2019compact}. In addition, methods such as POMCP (Partially Observable Monte Carlo Planning~\cite{silver2010monte,lim2019sparse}), leverage sampling-based approaches to estimate belief states and approximate the value function.  These approaches are not easy to compare to the approach in this paper since our approach is tailored explicitly to POMDPs derived from anticipation games for oblivious adversaries. Our approach is closely related to those that group  belief states together using bisimulation quotients~\cite{castro2009equivalence,castro2009notions,hermanns2014probabilistic}. A key distinction is that the approach presented here is an approximate notion of bisimulation wherein we guarantee that our information state machines track the precise belief state within a distance of $\lambda$ in a suitable norm. Thus, we exploit the special structure of the games studied here and prove that finite approximate bisimulations always exist for suitable choice of the parameters. %% game-based abstractions~\cite{winterer2020strategy} and factored representation~\cite{hansen2000dynamic}. 
%%%%%

While traditional POMDP solvers often work with the belief space, there have been approaches that leverage historical information to make decisions, either by directly maintaining a history or by approximating it. The complexity of solving this problem grows exponentially with the length of history~\cite{kearns1999approximate}. The results in~\cite{beauquier1995complexity} discuss this issue and address the trade-offs between memory usage and solution quality. To overcome this issue,~\cite{kaelbling1998planning}, introduces the concept of finite-memory controllers. In another work,~\cite{mccallum1995instance} investigates an instance-based learning approach for POMDPs, maintaining a set of histories to guide action selection. Similarly,~\cite{holmes2006looping,daswani2013feature} use looping suffix trees to represent the hidden state in deterministic finite POMDPs. This work is later extended to~\cite{meuleau2013solving}, which fixes the size of the policy graph to find the best policy of this size, and~\cite{meuleau2013learning}, that performs stochastic gradient descent on finite-state controller parameters, which guarantees local optimality of the solution. However, note that none of these techniques provide guarantees on the quality of the approximation or the solution so obtained. In this paper, we obtain such guarantees but for the limited case of POMDPs arising from the anticipation games and oblivious adversaries.

Our approach is an instance of the approximate information state introduced by Subramanian et al~\cite{subramanian2022approximate}, as a compression of history which is sufficient to evaluate approximate performance, and predict itself.  Yang et al~\cite{yang2022discrete} specialize this framework to  discrete approximate information states but their work learns the automaton from finite samples by solving an expensive nonlinear optimization problem. In this paper, we assume knowledge of the underlying game and opponent policies to construct a finite state machine that is guaranteed to be an approximation information state generator.

\section{Problem Definition}\label{sec:preliminaries}
A \emph{probability distribution} $d \colon X {\to} [0, 1]$  over a finite set $X$ satisfies $\sum_{s \in X} d(s) = 1$.
Let $\D(X)$ represent the set of all probability distributions over $X$. The distribution $d$ over $X = \setof{x_1, x_2, \cdots, x_m}$ is written $\setof{x_1: p_1, \ldots, x_m: p_m}$ where $p_i = d(x_i)$ for  $i \in [m]$.
For a natural number $n \geq 1$, let $[n]=\{ 1, 2, \ldots, n \}$. 
Bold case letters denote vectors $\vb \in \reals^n$. 
The $i^{th}$ component of $\vb$ is denoted as $b_i$.

% \subsection{Concurrent Stochastic Games Against Oblivious Opponents}
A \emph{Markov decision process} (MDP) $\scr{M}$ is a tuple $\tupleof{S, A, P, R}$ where $S$ is a finite set of states, $A$ is a finite set of actions, $P: S \times A \rightarrow \D(S)$ is the probabilistic transition function, and $R: S \times A \rightarrow \reals$ is a scalar valued reward function.
We write  $P(s' | s,a )$  for the probability of state $s'$ if action $a$ is applied to state $s$.
% \begin{definition}[Markov Decision Process]
% A Markov Decision Process (MDP) $\scr{M}$ is a tuple $\tupleof{S, A, P, R}$ wherein:
% \begin{compactitem}
% \item $S$ is a finite set of states, 
% \item $A$ is a finite set of actions,
% \item $P: S \times A \rightarrow \D(S)$ is the probabilistic transition function, and 
% \item $R: S \times A \rightarrow \reals$ is a scalar valued reward function.
% % \item $P: S \times A \rightarrow \D(S)$ maps each state and action to a distribution over next states.
% % \item $R: S \times A \rightarrow \reals$ maps each state and action to a real-valued reward.
% \end{compactitem}
% \end{definition}
% For MDP above, we write  $P(s' | s,a )$  for the probability of state $s'$ if action $a$ is applied to state $s$.
%Moreover, we define a \emph{path} in $\M$ as a sequence of state $s_0, s_1, \hdots, s_n$, where $P(s' | s,a ) > 0$. 
In a two player concurrent game, the set of actions are partitioned between player $\player1$ and $\player2$. 
Transitions of the game are determined by joint actions of both players. 
\begin{definition}[Concurrent Stochastic Game Arena: Syntax]
A concurrent stochastic game arena $\scr{G}$ is a tuple $\tupleof{S, \Ac1, \Ac2, P, R}$ wherein 
$S$ is a finite set of states,
$\Ac1$ and $\Ac2 $ are disjoint sets of actions for players $\player1$ and $\player2$, respectively,
$P: S \times \Ac1 \times \Ac2 \rightarrow \D(S)$ is the joint probabilistic transition function, and
$R: S \times \Ac1 \times \Ac2 \rightarrow \reals$ is a reward function for $\player1$.
\end{definition}
% \begin{definition}[Concurrent Stochastic Game Arena: Syntax]
% A concurrent stochastic game arena $\scr{G}$ is a tuple $\tupleof{S, \Ac1, \Ac2, P, R}$ wherein 
% \begin{compactitem}
%     \item $S$ is a finite set of states,
%     \item $\Ac1$ and $\Ac2 $ are disjoint sets of actions for players $\player1$ and $\player2$, respectively,
%     \item $P: S \times \Ac1 \times \Ac2 \rightarrow \D(S)$ is the joint probabilistic transition function, and
%     % \item $P: S \times \Ac1 \times \Ac2 \rightarrow \D(S)$ maps each state and joint actions for players $\player1$ and $\player2$ to a distribution over next states. 
%     \item $R: S \times \Ac1 \times \Ac2 \rightarrow \reals$ is a reward function for $\player1$.
%     % \item $R: S \times \Ac1 \times \Ac2 \rightarrow \reals$ is a reward map for $\player1$ specifying that the reward $R(s, a_1, a_2)$ obtained for the player 1 when $\player1$ plays action $a_1 \in A_1$ and $\player2$ concurrently plays action $a_2 \in \Ac2$ for state $s \in S$.
% \end{compactitem}
% \end{definition}
%We call a game arena $\scr{G} {=} \tupleof{S, \Ac1, \Ac2, P, R}$ an MDP if $\Ac2$ is singleton.
%Note that we do not define the reward function for $\player2$ as we do not assume that $\player2$ is adversarial or cooperative. 
We  assume that player $\player2$ selects their policy from one of $n$ different stochastic policies from the set $\Pi = \{ \pi_1, \ldots, \pi_n \}$, wherein each 
$\pi_i : S	{\rightarrow} \D(\Ac2)$
represents a map from states to probability distributions over actions in $\Ac2$. Let $\pi_i(s,a)$ denote the probability that action $a$ is chosen from state $s$ for policy $\pi_i$. 

% We assume $n \geq 2$ or else, the game simply becomes that of solving an  MDP for $n =1$.

\begin{example}\label{ex:rps-example}
Consider the RPS example discussed in the introduction (Figure~\ref{fig:rps-rewards-and-policies}).
The state set is a singleton: $S = \{t\}$. We have three actions each for players $1, 2$: $\Ac1 = \{ r_1, p_1, s_1 \}$ and $\Ac2 = \{ r_2, p_2, s_2 \}$, corresponding to choices of ``rock'', ``paper'' and ``scissors'', respectively. The transition probabilities are simply $P(t | t, a, b)  = 1$ for all $a \in \Ac1, b \in \Ac2$.  The reward  for $\player1$ is the familiar one from the game of rock-paper-scissors, and is shown in Figure~\ref{fig:rps-rewards-and-policies} (left)
$\player2$ plays one of four possible policies shown in the middle table  of Figure~\ref{fig:rps-rewards-and-policies}.
\end{example}
\begin{assumption}[Observation and Obliviousness]
\label{assum:oblivious}
We assume that:
\begin{inparaenum}[(a)]
\item $\player1$  observes the \emph{past} actions of $\player2$ but the \emph{current} action of one player is not observable by the other. 
\item $\player2$ is restricted to playing one of the policies $\set{\pi_1, \ldots, \pi_n}$ but this choice is \emph{not observable} by $\player1$. 
\end{inparaenum}
\end{assumption}

\paragraph{Policy Change Model:} We assume that  $\player2$ can change policies at each step depending on their current policies according to a Markov chain with $n$ states labeled by the corresponding policies $\pi_1, \ldots, \pi_n$.  Let $T$ represent the transition matrix of this Markov chain such that the entry $T_{ij} = P(\pi_j | \pi_i)$ represents the probability of $\player2$ switching their policy to $\pi_j$ given that their current policy is $\pi_i$.  Returning to Example~\ref{ex:rps-example}, the Markov chain for switching between the four policies $\pi_{1}, \ldots, \pi_4$ is shown in Figure~\ref{fig:rps-rewards-and-policies} (right). 

\ifextendedversion 
 % \subsection{Partially Observable MDP Semantics for Oblivious Game Arena (OGA)} 
\noindent A partially observable MDP (POMDP) is a tuple $\tupleof{S, A, P, R, \Omega, O}$ where $\tupleof{S, A, P, R}$ is an MDP, $\Omega$ is a finite set of \emph{observations}, and $O : S \to \Omega$ is (deterministic) observation map.
 The semantics of an OGA under Assumptions~\ref{assum:oblivious} can be given as a POMDP.
\begin{definition}[OGA: Semantics]
    The semantics of an OGA $\scr{G} = \tupleof{S, \Ac1, \Ac2, P, R}$ with player $\player2$ policy set $\set{\pi_1, \ldots, \pi_n}$ and policy change given by a Markov chain with transition matrix $T$ is a partially observable MDP (POMDP) $\scr{M'} = \tupleof{S', A' = \Ac1, P', R', \Omega = S, O}$ where
    \begin{itemize}
        \item $S' = S  \times [n]$ wherein each state $(s_i,j)$ represents a state $s_i \in S$ and an index $j \in [n]$ representing the current policy being employed by $\player2$
        \item  The probability of a transition 
 $P( (s', j') | (s, j), a)$ is given as:
%  is constructed in two steps as follows. We recall the assumption of decaying commitment: 
% \[ P(j'  | j ) = \begin{cases} 
% (1 - \epsilon) & j = j' \\ 
% \frac{\epsilon}{n-1} & \text{otherwise}\\
% \end{cases} \]
% Combining, we have the overall transition probability of the POMDP as 
% \[ P( (s', j') | (s, j),  a) =  \sum_{a_2 \in \Ac2} \pi_j(s, a_2) \times P(s'| s, a, a_2) \times P(j' | j) \,.\]
\[
P( (s', j') | (s, j),  a) = 
  T_{jj'} \cdot \sum\limits_{a_2 \in \Ac2} \left( \pi_j(s, a_2) \cdot P(s'| s, a, a_2) \right)
\]
\item The reward function is  given as:
$
R((s, j), a) = 
\sum_{a_2 \in \Ac2} \pi_j(s, a_2) \cdot R(s, a, a_2), \text{ and }
$
\item  The observation map $O: S' \to \Omega$ is defined as $(s_i,  j) \in S' \mapsto s_i$. 
    \end{itemize}
\end{definition}
 
While translating into a POMDP allows us access to a variety of approaches to solving POMDPs \cite{cassandra1998survey}, they are computationally expensive and ignore the specialized structure of the problem at hand. 
In this paper, we will work with the original two player game setup to directly  exploit the special problem structure at hand. 
\fi

%\subsection{Discounted Reward Optimization for OGA}
Our goal is to compute a \emph{finite memory} policy $\pi^{(1)}: S \times M \mapsto \Ac1$ that maximizes the expected discounted reward for $\player1$ with given discount factor $0<\gamma < 1$. 
The structure and construction of the required memory $M$  over the states and actions of $\player2$ is discussed in subsequent sections. 
% Our goal is to compute a policy $\pi: S \times M \mapsto \Ac1$ for $\player1$, where $M$ represents finite amount of memory that $\player1$ can retain about past actions and states of $\player2$ so that we maximize the expected discounted reward for a discount factor $\gamma < 1$. We will discuss the structure of this memory $M$ as a finite state machine over the states and actions of $\player2$ in this paper. 

\section{Information State Machine and Consistency}\label{sec:checking-consistency}
%\input{consistency-checking.tex}

% \subsection{Information State Machines}
The main approach  is to use a sequence of observations of states and $\player2$ actions to infer a \emph{belief state}  $\vb$ over the player's policies. 

\begin{definition}[Belief State]
A belief state $\vec{b}: (b_1, \ldots, b_n) \in \reals^n$ is a vector wherein the $i^{th}$ component  $b_i$ represents $\player1$'s belief that $\player2$ is employing policy $\pi_i \in \Pi$. Note that $b_i \geq 0$ for all $i \in [n]$ and $\sum_{i=1}^n b_i  = 1$.  
 \end{definition}

Let $\Beliefs_n = \{ \vb \in \reals^n\ |\  (\forall i \in [n])\ b_i \geq 0  \ \land\ \sum_{i=1}^n b_i = 1 \}$ denote the set of all belief state vectors in $\reals^n$.
The uniform belief state $\vec{b_u}$ is given by $(\frac{1}{n}, \ldots, \frac{1}{n})$. We define two operations over a belief state: (a) conditioning a belief state given some observation and (b) capturing the effect of policy change on a belief state.

Let $\vb$ be a belief state 
and $(s, a_2)$ represent an observation where $s \in S$ and $a_2 \in \Ac2$ represent states of the game and actions for $\player2$. The belief state $\vb' = \cond(\vb, s, a_2)$  is obtained by conditioning $\vb$ on the observation $(s, a_2)$: 
\begin{equation}\label{eq:conditioning}
b'_i = \cond(\vb, s, a_2) = \frac{ \pi_i(s, a_2)  b_i } { \sum_{j=1}^n \pi_j(s, a_2)  b_j  } \,.
\end{equation}
This expression is obtained as a direct application of Bayes' rule. %wherein the numerator represents the likelihood of observing action $a_2$  given that $\pi_i$ is the policy being used at state $s$ and the denominator is a normalizing constant.  

\begin{remark}\label{rem:belief-state-inconsistent}
The denominator in Eq.~\eqref{eq:conditioning} needs to be non-zero for $\cond(\vb, s, a_2)$ to be defined. The denominator being zero means that the current belief states rule out the observation $a_2$ as having zero probability.
 \end{remark}

 At each step, $\player2$ switches to a different policy from the one they are currently utilizing according to the Markov chain with transition probabilities given by $T$. This modifies a belief state $\vb$ to a new one $\vb' = T^t \vb$. I.e, 
$b'_i =  \sum_{j=1}^n b_j T_{ji} $~\footnote{We multiply by $T^t$ since traditionally Markov chains model probability distributions as row vectors. Left-multiplying  the transition matrix computes the distribution in the next step. }.

Overall, given a sequence $(t_1, a_1) (t_2, a_2) \cdots (t_k, a_k)$ of observations and starting from some initial belief state $\vb_0$, we define the sequence of belief states: 
$\vb_0 \xrightarrow{(t_1, a_1)} \vb_1 \xrightarrow{(t_2, a_2)} \vb_2 \cdots \xrightarrow{(t_k, a_k)} \vb_k$,
such that $\vb_{i+1} = T^t  \cond(\vb_i, t_{i+1}, a_{i+1})$, for 
$i \in [k-1]$. Recall the  \emph{total variation} (tv) distance between two belief states $\vb$ and $\vb'$, denoted
 $||\vb - \vb' ||_{tv} = \sum_{i=1}^n |b_i - b'_i| $.

%\begin{remark}
%The lemma shows that $\decay$ is ``contractive'' over the belief states as long as $n > 1$. Note that if $\frac{n\epsilon}{n-1} = 1$ or $\epsilon = \frac{n-1}{n}$, then $\decay(\vb) = \vb_u$ for all $\vb$, i.e,  $\player2$ effectively chooses uniformly at random from $\Pi$  at each step. 
%\end{remark}

We now discuss our model of history in terms of a finite state machine  over the states and alphabets of $\player2$ called the \emph{information state machine}.

\begin{definition}[Information State Machine]
An information state machine (ISM) is a deterministic finite state machine that consists of a  finite set of states $M$, alphabet $\Sigma = S \times \Ac2$, initial state $m_0$,  transition function $\delta: M \times \Sigma \rightarrow M$ and a map  that associates  state $m \in M$ with a belief state $\vb(m )$ with $\vb(m_0) = \vb_u$.
\end{definition}
The transition function can be extended to $\delta: M {\times} \Sigma^* \rightarrow M$ as\footnote{We write $\langle \text{empty} \rangle$ for an empty sequence and use $\circ$ for sequence concatenation.}
\[
\delta(m,\langle \text{empty} \rangle) = m, \text{ and } \delta(m, \sigma \circ (t, a)) = 
\delta(\delta(m, \sigma), (t, a)) \text{ for $\sigma {\in} \Sigma^*$ and $(t, a) {\in} \Sigma$.}
\]

% \[
% \delta(m,\sigma) = 
% \begin{cases}
% m & \text{if $\sigma = \langle \text{empty} \rangle$}\\
% \delta(\delta(m, \sigma'), (t, a)) & \text{if $\sigma = \sigma' \circ (t, a)$}
% \end{cases}
% \]
% For a sequence of observations $\sigma = (t_0, a_0) \cdots (t_k, a_k)$ and a state $m \in M$
% we let $\delta(m,\sigma) = \delta( \cdots \delta(\delta(m, (t_0, a_0)), (t_1, a_1)), \cdots, (t_k, a_k)) $.
The definition requires the state-machine to be deterministic.
% and therefore, for any state $m \in M$ and any observation $(t, a)$, we require the  $\delta(m, (t, a)) \in M$. 
However, we can relax this requirement to make $\delta$ a partial function. We require that for any sequence of observations $ \sigma: (t_0, a_0) \cdots (t_l, a_l)$, if $\sigma$ can occur with non-zero probability (i.e, there exist
actions $a_0', \ldots, a_{l-1}' \in \Ac1$, such that $ P(t_{j+1} | t_j, a_j', a_j)  > 0$ for all $j \in [l{-}1]$), 
then  (a unique state) $\delta(m_0, \sigma)$ must exist.

\begin{figure}[t]
\begin{center}
    \scalebox{0.8}{
  \begin{tikzpicture}[x=0.4mm, y=0.3mm, every_node/.style={fill=white}, my_node/.style={circle, draw=black, thin}]
  \begin{scope}
      \draw
        (27.0, 104.82) node[my_node] (0){0}
        (87.746, 148.3) node[my_node]  (1){1}
        (74.784, 18.0) node[my_node]  (2){2}
        (92.19, 88.257) node[my_node]  (3){3}
        (181.93, 71.262) node[my_node]  (4){4}
        (34.976, 34.738) node[my_node]  (5){5}
        (143, 34) node[my_node]  (6){6}
        (153.4, 130.09) node[my_node]  (7){7};
      \begin{scope}[->, every node/.style={fill=white}]
        \draw[red] (0) to (1); %node {$r_2$} (1);
        \draw[dashed] (0) to  (2); %{$s_2$} (2);
        \draw[blue, thick] (0) to (3); % node {$p_2$} (3);
        \draw[red] (1) to (6); %%node {$r_2$} (6);
        \draw[dashed] (1) to (7); %%node {$s_2$} (7);
        \draw[blue, thick] (1) to (3); %node {$p_2$} (3);
        \draw[red] (2) to (6); %node {$r_2$} (6);
        \draw[dashed] (2) to (5); %node {$s_2$} (5);
        \draw[blue, thick] (2) to (3); %node {$p_2$} (3);
        \draw[red] (3) to [bend left] (4); %%node {$r_2$} (4);
        \draw[dashed] (3) to[bend left] (5); %%node {$s_2$} (5);
        \draw[loop below, blue, thick] (3) to (3); %node {$p_2$} (3);
        \draw[red] (4) to (6); % node {$r_2$} (6);
        \draw[dashed] (4) to (7); %node {$s_2$} (7);
        \draw[blue, thick] (4) to[bend left] (3); %% node {$p_2$} (3);
        \draw[red] (5) to (6); %%node {$r_2$} (6);
        \draw[loop left, dashed] (5) to (5); %%node {$s_2$} (5);
        \draw[blue, thick] (5) to[bend left] (3); %%node {$p_2$} (3);
        \draw[loop below,red] (6) to (6); %% node {$r_2$} (6);
        \draw[dashed] (6) to[bend left] (7); %%node {$s_2$} (7);
        \draw[blue, thick] (6) to (3); %% node {$p_2$} (3);
        \draw[red] (7) to[bend left] (6); %%node {$r_2$} (6);
        \draw[loop right, dashed ] (7) to (7); %%node {$s_2$} (7);
        \draw[blue, thick] (7) to (3); %%node {$p_2$} (3);
      \end{scope}
      \end{scope}
      \begin{scope}[xshift=10cm, yshift=2.5cm]
      \draw (0,0) node {
      \small
      \begin{tabular}{|c|c|}
      \hline 
      State & Belief \\ 
      \hline 
      0 & (0.25, 0.25, 0.25, 0.25) \\ 
      1 & (0.29, 0.17, 0.29, 0.25) \\ 
      2 & (0.29, 0.29, 0.17, 0.25) \\ 
      3 & (0.17, 0.29, 0.29, 0.25) \\ 
      4 & (0.25, 0.17, 0.32, 0.26) \\ 
      5 & (0.26, 0.32, 0.17, 0.25) \\ 
      6 & (0.31, 0.17, 0.26, 0.26)\\
      7 & (0.33, 0.25, 0.17, 0.25) \\ 
      \hline 
      \end{tabular}
      };
      \end{scope}
    \end{tikzpicture}
    }
    \end{center}
    \caption{Example ISM for the rock-paper-scissors game.  Thick blue edges correspond to the observation $(0, p_2)$, dashed edges $(0, s_2)$ and solid red edges $(0, r_2)$.}
    \vspace{-1.5em}
    \label{fig:info-state-machine-example}
\end{figure}
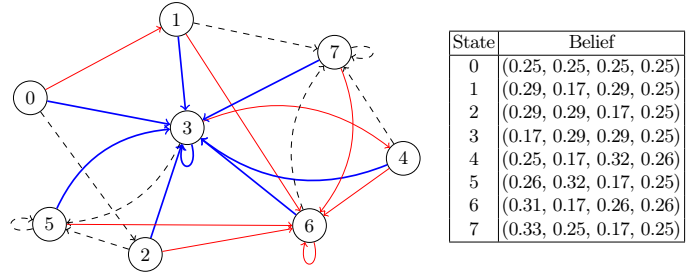

\begin{example}\label{ex:automata-example-1}
Figure~\ref{fig:info-state-machine-example} shows an example of an ISM for the rock-paper-scissors problem. Since $S$ has just one state, we do not include the label of this state in our alphabet, but simply label the edges with the actions of $\player2$. The initial  state is $0$ and the automaton is deterministic. 
\end{example}

We now define the  notion of consistency of an ISM. For any sequence of observations $\sigma:\ (t_0, a_0) \cdots (t_k, a_k)$ that can occur with positive probability, and a belief state $\vb \in \Beliefs_n$, let $\tau(\vb, \sigma)$ denote the result of \emph{transforming} $\vb$ successively based on the observations in $\sigma$. 
\[ \tau(\vb, \langle\text{empty}\rangle) = \vb,\ \text{and}\ \tau(\vb, \sigma\circ(t_i, a_i) ) = T^t  \cond(\tau(\vb, \sigma), (t_i, a_i)) \,.\]
% wherein  $\sigma \circ (t_i, a_i)$ is obtained by  adding $(t_i, a_i)$ to the end of $\sigma$. 

\begin{definition}[Consistent Information State Machine]\label{def:state-machine-consistent}
An ISM $\scr{M}$ is $\lambda$-consistent for  $\lambda > 0$ iff for every finite, positive probability sequence of $\player2$ state/action observations $ \sigma: (t_0, a_0) \cdots (t_k, a_k)$ such that
$m_{k+1} = \delta(m_0, \sigma)$, then the belief state $\vb(m_{k+1})$ remains sufficiently close to $\tau(\vb_u, \sigma)$, the belief state obtained from the full history: $ ||\vb(m_{k+1}) - \tau(\vb_u, \sigma) ||_{tv} \leq \lambda$. 
\end{definition}

 The concept of $\lambda$-consistency implies that for any history of observations of $\player2$'s actions, the belief state associated with the information state $m$ reached, remains within total-variation distance $\lambda$ of the belief state obtained by remembering the entire history.

\subsection{Consistency Checking}

In this subsection, we describe how to check whether a given ISM $\scr{M}$ is consistent for some limit $\lambda$ using the sufficient condition of edge consistency.

\begin{definition}[Edge Consistency]\label{def:edge-consistency}
An $e: m \xrightarrow{o} m'$ of the automaton $\scr{M}$ (i.e, $m, m' \in M$ and $\delta(m, o) = m'$) is \emph{consistent} for limit $\lambda$ iff 
\begin{equation} \label{eq:cond-consistency}
\forall\ \vb \in \Beliefs_n:\  \left(\sum_{j=1}^n b_j \pi_j(o) > 0 \ \land\  ||\vb - \vb(m)||_{tv} \leq \lambda \right)\ \Rightarrow\ ||\tau(\vb, o) - \vb(m')||_{tv} \leq \lambda \,.
\end{equation}
I.e, any belief state $\vb$ that is within a total variation distance $\lambda$ of $\vb(m)$ must, upon updating with observation $o$, yield a belief state $\tau(\vb, o)$ that is within $\lambda$ distance of $\vb(m')$. 
\end{definition}
Notice that we require $\sum_{j=1}^n b_j \pi_j(o) = P(o | \vb)$ to be positive. Failing this condition, the observation $o$ would be zero probability under the belief state $\vb$ and thus ruled out.
\begin{theorem}\label{thm:edge-consistency-equals-consistency}
    If every edge in an ISM $\scr{M}$ is edge consistent for limit $\lambda$ then the state machine is $\lambda$-consistent. 
\end{theorem}
\ifextendedversion
\begin{proof}
Following Def.~\ref{def:state-machine-consistent}, we need to show that for any finite sequence of observations $\sigma$, if $\delta(m_0, \sigma) = m$ then $||\vb(m) - \tau(\vb_u, \sigma)||_{tv} \leq \lambda$. 

Proof proceeds by induction on the length of the sequence $\sigma$, denoted $|\sigma|$. When $|\sigma| = 0$,  we have $m = m_0$ and $\tau(\vb_u, \sigma) = \vb_u$. Therefore, $||\vb(m) - \tau(\vb_u, \sigma)||_{tv} = 0 \leq \lambda$ holds. 

Assume that the result holds for any non-zero probability sequence $\sigma$ of length $m$. Let $\sigma' = \sigma \circ (t, a)$ for $t \in S $ and $a \in \Ac2$ also of non-zero probability. Let $m = \delta(m_0, \sigma)$ and $m' = \delta(m, (t,a))$. Since the observations, $\sigma$ and $\sigma'$ are assumed non-zero probability observations, we note the states $m, m'$ exist and are unique. We know by induction hypothesis that $||\vb(m) - \tau(\vb_u, \sigma)||_{tv} \leq \lambda$. Note that by edge consistency of the edge $m \xrightarrow{(t,a)} m'$, we have that  for all belief states $\vb \in \Beliefs_n $, we have 
\[ ||\vb - \vb(m)||_{tv} \leq \lambda \ \Rightarrow\ ||\tau(\vb, (t,a)) - \vb(m') ||_{tv} \leq \lambda \,.\]
Applying this to $\vb = \tau(\vb_u, \sigma)$, we note that the antecedent holds by induction hypothesis and thus, we conclude that 
\[ || \underset{= \tau(\vb_u, \sigma')}{\underbrace{\tau(\tau(\vb_u, \sigma), (t,a)) - \vb(m')||_{tv}}} \leq \lambda \,.\]

\end{proof}
\else 
Proof is by induction on the size of the observation sequences, and is given in the extended version of this paper.
\fi
We now provide an approach to check if a given edge in an automaton $e:\ m \xrightarrow{o} m'$ is consistent for a limit $\lambda$ by checking a formula in linear arithmetic.  We will attempt to find a
belief state $\vb$ that \emph{refutes}~\eqref{eq:cond-consistency}.
I.e, $\vb \in \Beliefs_n $ that satisfies conditions: (a) $||\vb - \vb(m)||_{tv} \leq \lambda$; (b) $\sum_{j=1}^n \pi(o) b_j > 0$ and (c) $||\tau(\vb, o) - \vb(m')||_{tv} > \lambda$. Note that $\vb(m)$ and $\vb(m')$ are known belief-vectors while $\vb$ is the unknown vector we seek. We will construct a formula $\Psi_{e}$ in linear arithmetic  such that edge $e$ is consistent iff $\Psi_e$ is unsatisfiable.  The formula $\Psi_e$ is encoded using variables $\vb:(b_1, \ldots, b_n)$ representing the unknown belief state and extra variables $\vx: (x_1, \ldots, x_n)$ and $\vy: (y_1, \ldots, y_n)$. Let $\alpha_i = \pi_i(o)$ represents the probability of observation $o$ under policy $\pi_i$.

\noindent \textbf{(1)}  Observation $o$ occurs with non-zero probability:
 \[ \Psi_0(e):\ \sum_{j=1}^n \alpha_j b_j > 0\ \land\ \sum_{i=1}^n b_i = 1 \,.\]

\noindent \textbf{(2)} $||\vb - \vb(m)||_{tv} \leq \lambda$ must hold.
\[ \Psi_1(e):\ \bigwedge_{i=1}^n x_i \geq 0\ \land\ \bigwedge_{i=1}^n \underset{\equiv |b_i - \vb(m)_i| \leq x_i}{\underbrace{ -x_i \leq (b_i - \vb(m)_i) \leq x_i}}\ \land\ \sum_{i=1}^n x_i \leq \lambda \,.\]
\noindent\textbf{(3)} $||\tau(\vb, o) - \vb(m')||_{tv} > \lambda$. Recall  $\tau(\vb,o) = T^t \times (\cond(\vb, o)) = T^t \times \left(\frac{b_1 \alpha_1}{\sum_{j=1}^n b_j \alpha_j}, \cdots,\frac{b_1 \alpha_1}{\sum_{j=1}^n b_j \alpha_j} \right)  $.  
\[ ||\tau(\vb,o) - \vb(m')||_{tv} = \sum_{j=1}^n \left|\frac{\sum_{i=1}^n T_{ij} \alpha_i b_i}{\sum_{i=1}^n \alpha_i b_i} - \vb(m')_j\right|\,.\]
Let $e_j$ denote the expression $ \left(\sum_{i=1}^n T_{ij} \alpha_i b_i\right) - \vb(m')_j \left(\sum_{i=1}^n \alpha_i b_i\right)$.
Since $\sum_{j=1}^n \alpha_j b_j > 0$, the condition $||\tau(\vb, o) - \vb(m')||_{tv} > \lambda$ is equivalent to 
\[ \Psi_2(e):\ \bigwedge_{j=1}^n  \underset{ \equiv\  y_j = |e_j|} {\underbrace{ y_j \geq 0\ \land\ (y_j = e_j \lor y_j = -e_j)}}\ \land\ \left( \sum_{j=1}^n y_j > \lambda \sum_{j=1}^n \alpha_j b_j\right) \,.\]

\begin{theorem}\label{thm:edge-consistency-feasibility}
An edge $e$ is consistent iff $\Psi(e):\ \Psi_0(e)\ \land\ \Psi_1(e)\ \land\ \Psi_2(e)$ is infeasible.
\end{theorem}

 SMT solvers such as Z3 can be used to check satisfiability~\cite{nieuwenhuis2006solving}. Alternatively, linear complementarity problem (LCP) solvers~\cite{murty1988linear} can also be used: disjunction $y_i {=} e_i \lor y_i {=} -e_i$ is equivalent to a complementarity constraint $(y_i - e_i) \perp (y_i + e_i)$.

\begin{comment}
\begin{figure}[t]
\begin{center}
\begin{tikzpicture}
\node(n0) at (-4,0) {\includegraphics[width=0.45\textwidth]{figures/consistent_example_edge1.png}};
\node(n1) at (2,0) {\includegraphics[width=0.45\textwidth]{figures/inconsistent_example_edge1.png}};
\draw (n0.south)++(0,-0.25cm) node {\small  $0 \xrightarrow{p_2} 3$};
\draw (n1.south)++(0,-0.25cm) node {\small  $4 \xrightarrow{r_2} 6$};
\end{tikzpicture}
\end{center}
\caption{Plots of total variation distance in the pre-state (x-axis) vs post-state (y-axis) for two types of edges for the automaton shown in Figure~\ref{fig:info-state-machine-example}: (left) edge that is proven consistent and (right) an edge that is not consistent. }\label{fig:tv-norms-pre-post}
\end{figure}
\end{comment}

\begin{example}
We check the consistency of the automaton from Example~\ref{ex:automata-example-1} for $\lambda = 0.25$. For the edge $e:\ 4 \xrightarrow{r_2} 6$ in the automaton.  The formula $\Psi(e)$ is satisfiable with $\vb = (0.125, 0.17, 0.445, 0.26)$:  $||\vb - \vb(4)||_{tv} = 0.25$, whereas $||\tau(\vb, o) - \vb(6)||_{tv} \approx 0.337  > 0.25$.
The automaton in Figure~\ref{fig:info-state-machine-example}
fails to be consistent. 
 \end{example}

%In the next section, we present an approach that upon successful termination produces a consistent ISM for given $\lambda$.

\section{Information State Machine Synthesis Algorithm}\label{sec:consistent-ism-synthesis}

\begin{algorithm}[t]
\DontPrintSemicolon
\KwData{$\scr{G}, \Pi, T, \lambda$}
\KwResult{A finite state machine $\scr{M}$.}
$m_0 \ \leftarrow\ \mathsf{newState}(\vb_u)$ \tcp*{create initial state}
$\Sigma' = \{ (s, a_2) \in S \times \Ac2\ |\ (\exists \pi \in \Pi)\ \pi(s, a_2) > 0 \} $\tcp*{non-zero prob. observ.}\label{nl:non-zero-prob}
$(\scr{M}, W) \leftarrow\ (\emptyset, [m_0])$  \tcp*{initialize set of states and worklist}
\While{$W \not= \emptyset$}{
    $m\ \leftarrow\ pop(W)$ \tcp*{pop a state from the worklist }
    Add state $m$ to $\scr{M}$\;
    \For{$o \in \Sigma'$ \tcp*{iterate through observations} \label{for-loop-head}}{ 
        $\vb' \leftarrow \tau(\vb(m), o)$ \tcp*{compute next belief state}\label{next-belief-state}
        \lIf(\tcp*{check consistency}){$\mathsf{not\ isConsistent}(\vb(m), o, \vb')$}{ \textbf{FAIL} \label{consistency-check-2}}
        $ \hat{m}\ \leftarrow\ \mathsf{findClosestState}(\vb', \lambda) $ \label{find-closest-state}    \tcp*{search for nearby state}
        \uIf{ $\hat{m} \not= \text{Nil}\ \land\ \mathsf{isConsistent}(\vb(m), o, \vb(\hat{m}) )$ \tcp*{ existing state found}\label{nearby-state-found}}{
            Add edge $m \xrightarrow{o} \hat{m}$  to $\scr{M}$\;
        }\Else{
                $m' = \mathsf{newState}(\vb')$\label{create-new-state} \tcp*{Create new information state}\;
                Add edge $m \xrightarrow{o} m'$ to $\scr{M}$\;
                $\mathsf{push}(m', W)$ \tcp*{push new state to worklist}\;
        } 
    }
    return $\scr{M}$ \;
}
\caption{\textsc{ConstructConsistentInformationStateMachine}()}\label{alg:info-state-machine-search}
\end{algorithm}

Algorithm~\ref{alg:info-state-machine-search} attempts to synthesize a consistent finite state machine for $\player2$, given a concurrent game $\scr{G}:\ \tupleof{S, \Ac1,\Ac2, P, R}$, policies $\Pi:\ \{\pi_1, \ldots, \pi_n\}$, transition matrix $T$ and $\lambda > 0$ by exploring belief states starting from the initial belief state $m_0$. Line~\ref{nl:non-zero-prob} restricts the alphabet to the set $\Sigma'$ that has non-zero probability under at least one policy.  The algorithm maintains a worklist $W$ that is initialized to contain the initial state $m_0$ at start. At each iteration, it pops a state from the worklist and adds it to the automaton. Next, the algorithm iterates through all the observations $o \in \Sigma'$ (line number~\ref{for-loop-head}). After computing the 
next belief state $\vb'$ (line~\ref{next-belief-state}), it finds the closest state to $\vb'$ in the total variation norm and checks that it is closer than the limit $\lambda$ (line~\ref{find-closest-state}). If such a state $\hat{m}$ is found and the edge from $m$ to $\hat{m}$ is consistent (line \ref{nearby-state-found}), then the edge is added. Consistency is checked using a SMT or MILP solver as described in Section~\ref{sec:checking-consistency}. Otherwise, the algorithm has already checked consistency of the new state and edge that it is about to create (line~\ref{consistency-check-2}). This is an important operation since a failure of consistency here can result in an overall failure to find a state machine. 
\begin{theorem}
Any automaton $\scr{M}$ returned by Algorithm~\ref{alg:info-state-machine-search}  is $\lambda$-consistent.
\end{theorem}
\begin{proof}
Every edge added to the automaton is consistent, by construction.
\end{proof}

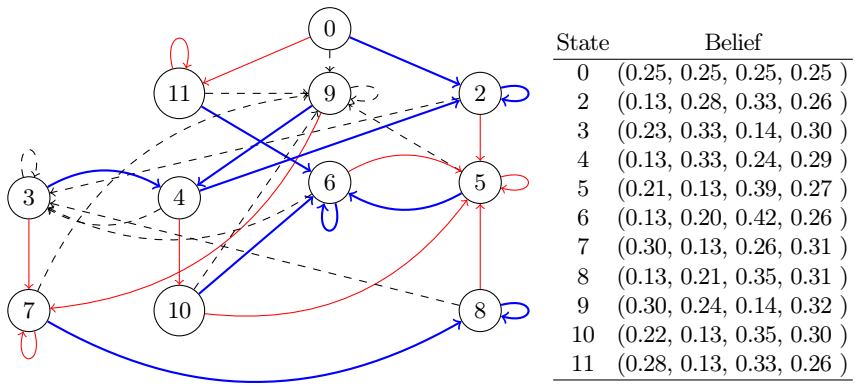
\begin{figure}[t]
\begin{tabular}{ll}
\begin{minipage}{0.5\textwidth}
\begin{tikzpicture}[x=4cm, y=2.5cm]
      \draw[every node/.style={circle, draw=black, thin}]
        (0.5, 1.1) node (0){0}
        (1, 0.756) node (2){2}
        (0.5, 0.756) node (9){9}
        (0, 0.756) node (11){11}
        (-0.5, 0.2) node (3){3}
        (1, 0.282) node (5){5}
        (0, 0.2) node (4){4}
        (-0.5, -0.4) node (7){7}
        (0, -0.4) node (10){10}
        (0.5, 0.282) node (6){6}
        (1, -0.4) node (8){8};
      \begin{scope}[->]
        \draw[thick, blue] (0) to (2); %node[fill=white] {\footnotesize $p_2$} (2);
        \draw[dashed] (0) to (9); % node[fill=white] {\footnotesize $s_2$} (9);
        \draw[red] (0) to (11); 
        %node[fill=white] {$\footnotesize r_2$} (11);
        \draw[thick, blue, loop right,] (2) to (2); %node[fill=white] {$\footnotesize p_2$} (2);
        \draw[dashed] (2) to (3) ; %node[fill=white] {\footnotesize $s_2$} (3);
        \draw[red] (2) to (5); 
        %node[fill=white] {\footnotesize $r_2$} (5);
        \draw[thick, blue] (9) to (4); 
        %node[fill=white] {\footnotesize $p_2$} (4);
        \draw[loop right, dashed] (9) to (9);  % node[below, fill=white] {$s_2$} (9);
        \draw[bend left, red] (9) to (7);
        %node[fill=white] {\footnotesize $r_2$} (7);
        \draw[thick, blue] (11) to  (6); 
        %node[fill=white] {\footnotesize $p_2$} (6);
        \draw[dashed] (11) to  (9);
        %node[fill=white] {\footnotesize $s_2$} (9);
        \draw[loop above,red] (11) to (11);  % node[fill=white] {\footnotesize $r_2$} (11);
        \draw[bend left, thick, blue ] (3) to  (4); %node[fill=white] {\footnotesize $p_2$} (4);
        \draw[loop above, dashed] (3) to (3);  % node[fill=white] {\footnotesize $s_2$} (3);
        \draw[red] (3) to (7);  % node[fill=white] {\footnotesize $r_2$} (7);
        \draw[bend left, thick, blue] (5) to (6);  %node[fill=white] {\footnotesize $p_2$} (6);
        \draw[dashed] (5) to (9); 
        %node[] {s2} (9);
        \draw[loop right,red] (5) to (5); 
        %node[fill=white] {\footnotesize $r_2$} (5);
        \draw[thick, blue] (4) to (2); % node[fill=white] {\footnotesize $p_2$} (2);
        \draw[bend left, dashed] (4) to (3); %node[fill=white] {\footnotesize $s_2$} (3);
        \draw[red] (4) to (10);
        %node[fill=white] {\footnotesize $r_2$} (10);
        \draw[bend right, thick, blue] (7) to (8); %node[fill=white] {$\footnotesize p_2$} (8);
        \draw[dashed, bend left] (7) to (9); 
        %node[fill=white] {\footnotesize $s_2$} (9);
        \draw[loop below,red] (7) to (7); 
        %node[fill=white] {\footnotesize $r_2$} (7);
        \draw[thick, blue] (10) to (6);
        %node[fill=white] {\footnotesize $p_2$} (6);
        \draw[dashed] (10) to (9); 
        %node[fill=white] {\footnotesize $s_2$} (9);
        \draw[red, bend right] (10) to (5);
        %node[fill=white] {\footnotesize $r_2$} (5);
        \draw[loop below, thick, blue] (6) to (6); %node[fill=white] {\footnotesize $p_2$} (6);
        \draw[bend left, dashed] (6) to (3);
        %node[fill=white] {\footnotesize $s_2$} (3);
        \draw[bend left, red] (6) to (5);
        %node[fill=white] {\footnotesize $r_2$} (5);
        \draw[loop right, thick, blue] (8) to (8); %node[fill=white] {\footnotesize $p_2$} (8);
        \draw[dashed] (8) to (3); 
        %node[fill=white] {\footnotesize $s_2$} (3);
        \draw[red] (8) to (5); 
        %node[fill=white] {\footnotesize $r_2$} (5);
      \end{scope}
    \end{tikzpicture}
\end{minipage} &
\begin{minipage}{0.5\textwidth}
\begin{center}
\begin{tabular}{cc}
State & Belief \\ 
\hline 
0 & (0.25, 0.25, 0.25, 0.25 )\\
2 & (0.13, 0.28, 0.33, 0.26 )\\
3 & (0.23, 0.33, 0.14, 0.30 )\\
4 & (0.13, 0.33, 0.24, 0.29 )\\
5 & (0.21, 0.13, 0.39, 0.27 )\\
6 & (0.13, 0.20, 0.42, 0.26 )\\
7 & (0.30, 0.13, 0.26, 0.31 )\\
8 & (0.13, 0.21, 0.35, 0.31 )\\
9 & (0.30, 0.24, 0.14, 0.32 )\\
10 & (0.22, 0.13, 0.35, 0.30 )\\
11 & (0.28, 0.13, 0.33, 0.26 )\\
\hline 
\end{tabular}
\end{center}
\end{minipage}
\end{tabular}
\caption{\textbf{(Left)} Consistent ISM for $\lambda=0.25$ the RPS example from Figure~\ref{fig:rps-rewards-and-policies} obtained by running Algorithm~\ref{alg:info-state-machine-search}. Thick blue edges correspond to the observation $(0, p_2)$, dashed edges $(0, s_2)$ and solid red edges $(0, r_2)$; \textbf{(Right)} Beliefs associated with states.}\label{fig:consistent-ism-rps-example}
\end{figure}

%We note that the algorithm can be improved considerably in many ways. Rather than adding an edge corresponding to each observation for each state, the implementation only adds  edges corresponding to those observations at a state of the automaton that can occur with a non-zero probability. 
Figure~\ref{fig:consistent-ism-rps-example} shows a consistent ISM with $11$ states for the RPS example from Figure~\ref{fig:rps-rewards-and-policies}.
Note that Algorithm~\ref{alg:info-state-machine-search} is not guaranteed to terminate and return a finite ISM. In section~\ref{sec:completeness-and-robustness}, we establish a simple condition on the transition matrix $T$ for which the algorithm terminates and yields a finite ISM.  
%We will also study the \emph{robustness} w.r.t $T$: I.e, if $\player2$ uses a transition matrix $T_A$ during play which is different from the matrix $T_D$ during design, we will show that the ISM remains consistent for a larger value of $\lambda$ provided some conditions on $T_A, T_D$ are met. 

\section{Policy Synthesis}\label{sec:policy-synthesis}
Given an ISM $\scr{M}$, we will now describe the policy synthesis for $\player1$ and prove bounds on the optimality of the policy thus obtained w.r.t discounted rewards. 
We first compose a two player game graph $\scr{G}: \tupleof{S, \Ac1, \Ac2, P, R}$ with the ISM $\scr{M}: \tupleof{M, \Sigma', \delta}$ wherein $\Sigma'\subseteq S \times \Ac2$. This MDP serves as a starting point for optimal policy synthesis. Next, for a $\lambda-$consistent information state machine. We show that this MDP is ``close'' to an infinite state MDP obtained from unbounded histories. We invoke a result on approximate information states (AIS) by Subramanian et al \cite{subramanian2022approximate} to  bound  the difference between the optimal value function obtained from  finite state histories and that from full histories. 

The MDP is given by $\tupleof{S \times M, \Ac1, \widehat{P}, \widehat{R}}$ with states $(s, m)$ for $s \in S$ and  $m \in M$. Let $\vb(m) = (b_1, \ldots, b_n)$.
%For a belief state $\vb \in \Beliefs_n$,  the probability of $\player2$ choosing action $a_2 \in \Ac2$ is
%\begin{equation}\label{eq:prob-player-2-action-given-belief}
%    P(a_2 | \vb) =  \sum_{i=1}^n b_i \pi_i(s, a_2)  \,.
%\end{equation}
%For action $a_2 \in \Ac2$,  $\scr{M}$ transitions from $m$ to $m'$, where $m' = \delta(m, (s, a_2))$. 
%\begin{equation}\label{eq:mdp-transition-prob-two-actions}
%\hat{P}( (s', m') | (s, m), a, a_2) =  P(s' | s, a, a_{2}) \left( \sum_{i=1}^n b_i \pi_i(s, a_2) \right)
%\end{equation}
For  $a_1 \in \Ac1$, the probability of transitioning to $(s', m')$ from $(s,m)$ is given by
% \begin{equation}\label{eq:mdp-transition-probabililty}
% \widehat{P}((s', m') | (s, m), a_1) = \sum_{ \stackrel{a_2 \in \Ac2}{\delta(m, (s,a_2)) = m'}}  \underset{= \mathbb{P}( a_2 | \vb(m) )}{\underbrace{\left( \sum_{i=1}^n b_i \pi_i(s, a_2) \right)}}  P(s' | s, a, a_{2})\,.
% \end{equation}
\begin{equation}\label{eq:mdp-transition-probabililty}
\widehat{P}((s', m') | (s, m), a_1) = \sum_{ a_2 \in \Ac2} \underset{\text{indicator function}}{\underbrace{{1}_{\{\delta(m, (s,a_2)) = m'\}}}}   \underset{= \mathbb{P}( a_2 | \vb(m) )}{\underbrace{\left( \sum_{i=1}^n b_i \pi_i(s, a_2) \right)}}  P(s' | s, a, a_{2})\,.
\end{equation}
Note that $1_{\{\psi\}} =1 \ \text{if}\ \psi\  \text{holds and}\ 0 \ \text{otherwise}$. 
The reward function is 
\begin{equation}\label{eq:mdp-reward-definition}
    \widehat{R}((s,m), a_1) = \sum_{a_2 \in \Ac2} \underset{\text{see eq.~\eqref{eq:mdp-transition-probabililty}}}{\underbrace{\mathbb{P}(a_2 | \vb(m))}} \  \underset{\text{from}\ \scr{G}}{\underbrace{R(s, a_1, a_2)}}\,.
\end{equation}

%\subsection{Discrepancy Analysis}
The composition of a finite ISM with the game yields a finite-state MDP for $\player1$ that can be solved to yield a policy for $\player1$. However, since the ISM tracks the belief state approximately, we cannot expect the resulting policy to be optimal when compared to a situation wherein we  track the precise belief state.  We will  bound the loss in value resulting from the belief state approximation in an ISM. 

 We construct a belief state MDP   using the ``exact'' history of observations up to some time $t$. The ``exact'' MDP has as its 
states $S \times \Beliefs_n $ wherein each state is a pair $(s, \vb^{(t)})$ for $s \in S$ and $\vb^{(t)} = \tau(\vb_u, \sigma_t)$ for observation sequence $\sigma_t: (s_1, a_1), \ldots, (s_t, a_t)$. The expected reward obtained for action $a \in \Ac1$ in current state $s_{t+1} = s$ is given by  
\begin{equation}\label{eq:ideal-reward-definition}
R^*((s, \vb^{(t)} ), a)  =  \sum_{a_2 \in \Ac2} P(a_2| \vb^{(t)}) \  R(s, a_1, a_2)\,.
\end{equation}

We define the  transition probability $P^*$ as 
\[ P^*( (s', \vb^{(t+1)}) | (s, \vb^{(t)}), a, a_{t+1}) =  1_{\{\vb^{(t+1)} = \tau(\vb^{(t)}, (s, a_{t+1}))\}} P(a_{t+1} | \vb^{(t)})  P(s' | s, a, a_{t+1}) \,.\]
Let $m_t = \delta(m_0, \sigma_t)$ be the unique information state from $\scr{M}$. Since $\scr{M}$ is $\lambda$-consistent, we know that $||\vb(m_t) - \vb^{(t)}||_{tv} \leq \lambda$. We establish bounds on the discrepancies between the rewards obtained and the next state probabilities. Let us define $R_{\max}(s) = \max_{a_1 \in \Ac1, a_2 \in \Ac2} | R(s, a_1, a_2) | $ and $\alpha_{\max}(s) = \sum_{a_2 \in \Ac2}\ \max_{j=1}^n \pi_j(s, a_2)$.
\begin{lemma}\label{lemma:reward-discrepancy}
For any history $\sigma_t$,  
$|R^*((s, \vb^{(t)}), a) - \widehat{R}((s,m_t), a)| \leq R_{max}(s) \alpha_{max}(s)  \lambda$.
\end{lemma}
%Proof is in Appendix~\ref{appendix:proof-reward-discrepancy}. 
%Corresponding to $\sigma_{t}, \sigma_{t+1}$ we have an associated states $m_t = \delta(m_0, \sigma_t)$ and  $m_{t+1} = \delta(m_0, \sigma_{t+1})$. 
\ifextendedversion
\begin{proof}
 We expand the LHS using Eq.~\eqref{eq:mdp-reward-definition} and Eq.~\eqref{eq:ideal-reward-definition}. 
\begin{align*}
|R^*((s, \vb^{(t)}), a) - \widehat{R}((s,m_t), a)| & \leq \sum_{a_2 \in \Ac2} | P(a_2 | \vb^{(t)}) R(s, a_1, a_2) - P(a_2 | \vb(m_t)) R(s, a_1, a_2)| \\ 
& \leq \sum_{a_2 \in \Ac2} |R(s, a_1, a_2)| | \sum_{i=1}^n \vb^{(t)}_i \pi_i(s, a_2) - \vb(m_t)_i \pi_i(s, a_2) | \\ 
& \leq | R_{\max}(s)| \sum_{a_2 \in \Ac2} (\max_{j=1}^n \pi_j(s, a_2))\sum_{i=1}^n |\vb^{(t)}_i - \vb(m_t)_i| \\ 
& \leq |R_{\max}(s)| \sum_{a_2 \in \Ac2} (\max_{j=1}^n \pi_j(s, a_2)) ||\vb^{(t)} - \vb(m)||_{tv} \\ 
& \leq |R_{\max}(s)|\  \alpha_{\max}(s)\  \lambda \\
\end{align*}
\end{proof}
\fi 
Next, we prove that the next-state distributions $\hat{P}$ and $P^*$ are ``close'' in the total-variation distance $d_{tv}(\sigma_t, s, a)$ given by the formula:
\[ \sum_{a_{t+1} \in \Ac2} \sum_{s' \in S} \left| P^*( (s', \vb^{(t+1)}) | (s, \vb^{(t)}), a, a_{t+1}) - \widehat{P}((s', m_{t+1}) | (s, m_t), a, a_{t+1}) \right| \,.\]
\begin{lemma}\label{lemma:transition-map-discrepancy}
For any history $\sigma_t$ and action $a \in \Ac1$, 
$d_{tv}(\sigma_t, s, a) \leq  \alpha_{\max}(s) \lambda$.
\end{lemma}
\ifextendedversion
\begin{proof}
We wish to bound the summation.
\[ \sum_{a_{t+1} \in \Ac2} \sum_{s' \in S} \underset{D}{\underbrace{| P^*( (s', \vb^{(t+1)}) | (s, \vb^{(t)}), a, a_{t+1}) - \widehat{P}((s', m_{t+1}) | (s, m_t), a, a_{t+1}) |}} \,.\]
Let $D$ denote the term inside the summation.

\begin{align*}
D  & \leq | P(a_{t+1} | \vb^{(t)})  P(s' | s, a, a_{t+1}) - P(a_{t+1} | \vb(m_t))  P(s' | s, a, a_{t+1}) | \\ 
& \leq P(s' | s, a, a_{t+1})   \max_{i=1}^n (\pi_i(s, a_{t+1}))  ||\vb^{(t)} - \vb(m_t)||_{tv} \\ 
& \leq P(s' | s, a, a_{t+1})   \max_{i=1}^n (\pi_i(s, a_{t+1}))  \lambda \\ 
\end{align*}

Using this, we can bound $d_{tv}(\sigma_t, s, a) $ as
\begin{align*}
d_{tv}(\sigma_t, s, a) & \leq \sum_{a_{t+1} \in \Ac2} \sum_{s' \in S} P(s' | s, a, a_{t+1} )  \max_{i=1}^n (\pi_i(s, a_{t+1}))  \lambda \\ 
& \leq  \lambda \sum_{a_{t+1} \in \Ac2}  \max_{i=1}^n (\pi_i(s, a_{t+1}))  \underset{=1}{\underbrace{\sum_{s' \in S} P(s' | s, a_{t+1}, a)}} \\ 
& \leq \lambda \alpha_{\max}(s)  \\ 
\end{align*}

\end{proof}
\fi 
%Proof is in Appendix~\ref{app:proof-transition-map-discrepancy}. 
For some discount factor $\nu$, let $V^*$ be the optimal value for the (infinite state) ``exact'' MDP with state-space $S \times \Beliefs_n $, actions $\Ac1$, transition relation $P^*$ and expected reward $R^*$. Let $\hat{V}$ be the optimal value function for the MDP 
with state space $S \times M$, transition map $\hat{P}$ and reward $\hat{R}$.

\begin{theorem}\label{thm:value-function-approximate}
There exists $K$ such that for each history $\sigma_t$ leading to belief $\vb^{(t)}$,  ISM state $m_{t}$ and for every game state $s$, we have 
$| V^*(s, \vb^{(t)}) - \widehat{V}(s, m_t) | \leq K \lambda$.
% \[ | V^*(s, \vb^{(t)}) - \widehat{V}(s, m_t) | \leq K \lambda \,.\]
\end{theorem}
This follows from Theorem 27 of Subramanian et al~\cite{subramanian2022approximate} where the constant $K$ equals
$\dfrac{| R_{\max}(s)| \alpha_{\max}(s)+ \gamma \rho}{1-\gamma},$
% \[ K= \dfrac{| R_{\max}(s)| \alpha_{\max}(s)+ \gamma \rho}{1-\gamma} \,,\]
where $\rho$ is the ``Lipschitz constant'' for the function $V$. We  conclude that a $\lambda-$consistent information state machine can be used in lieu of an exact belief state with a loss in value  proportional to $\lambda$.

\section{Completeness and Robustness}\label{sec:completeness-and-robustness}
In this section, we first provide a sufficient condition on the transition matrix $T$ that governs how $\player2$ switches between policies so that Algorithm~\ref{alg:info-state-machine-search} is guaranteed to terminate successfully and yield a finite ISM. 
Let $t^*$ be such that for all $i, j \in [n]$, $T_{ij} \geq t^*$. I.e, $t^*$ is the smallest entry in the matrix $T$. We assume that $t^* > 0$: i.e, the transition matrix $T$ is strictly positive. Note that  the entries for each row of $T$ sum up to $1$. Therefore, $t^* \leq \frac{1}{n}$. 
Let $\vb = \tau(\vb_0, \sigma)$ be the exact belief state obtained starting from the uniform initial belief state $\vb_0$ and a sequence of non-zero probability observations $\sigma$. 
\begin{lemma}\label{lemma:belief-lower-bounds}
Each entry of $\vb$ satisfies $b_j \geq t^*$.
\end{lemma}
\begin{proof}
Proof is by induction on the length of the sequence $\sigma$. The base case holds for $\vb = \vb_0$ since $b_{0,j} = \frac{1}{n} \geq t^*$. Let $\vb = \tau(\vb_0, \sigma)$ for $|\sigma| = n$. Let $o$ be an observation such that $\vb' = \tau(\vb, o)$. By induction hypothesis, $b_j  \geq t^*$. We have $\vb'  = T^t \hat{\vb}$ where $\hat{\vb} = \cond(\vb, o)$ is a belief vector.
$b'_j = \sum_{i=1}^n T_{ij} \hat{b}_i  \geq  t^* \sum_{i=1}^n \hat{b}_i \geq t^*$.
\end{proof}

For observation $o$, let $\alpha_j = \pi_j(o)$, $\alpha_{\max}(o) = \max_{j=1}^n \alpha_j$ and $\alpha_{sum}(o) = \sum_{j=1}^n \alpha_j$.  We define $\kappa(o) = \frac{\alpha_{\max}(o)}{\alpha_{sum}(o) + n \alpha_{\max}(o)}$.  Let $\kappa_{\max} = \max_{o \in \Sigma \times \Ac2} \kappa(o)$.

\begin{theorem}\label{thm:completeness-ism-algo}
If  $t^* > \kappa_{\max}$, then for any parameter $\lambda > 0$, Algorithm~\ref{alg:info-state-machine-search} terminates successfully to yield a finite state consistent ISM.
\end{theorem}
We first provide a sketch of the proof.  (a) We first establish that the function $\vb \mapsto \tau(\vb, o)$ is \emph{contractive} in the total variation norm whenever $t^* > \kappa(o)$. Therefore, the consistency check in line~\ref{consistency-check-2} will always succeed, or equivalently, Algorithm~\ref{alg:info-state-machine-search}  will not return \textbf{FAIL}. It remains to show that the Algorithm will terminate. (b) Next, we show that whenever the call to $\mathsf{findClosestState}(\vb', \lambda)$ in line~\ref{find-closest-state} yields a state $\hat{m}$ such that $\vb(\hat{m})$  is within distance $(1 - \kappa_{\max}) \lambda$ of $\vb'$, then the edge $m \xrightarrow{o} \hat{m}$ will be consistent. Therefore, we show that for any new state created by Algorithm~\ref{alg:info-state-machine-search} line~\ref{create-new-state}, the total variation distance from any previously created state is at least $(1 - \kappa_{\max}) \lambda$. (c) The number of states in the ISM is therefore bounded by the \emph{packing number of the compact set}  $\Beliefs_n$ with $L_1$ norm balls of radius $(1 - \kappa_{\max}) \lambda$~\cite{Mohri+Others/2012/Foundations}. 
\ifextendedversion
\paragraph{Proof.} Let us assume that $T_{i,j} \geq t^* $ for all $i, j \in [n]$.
We will first derive conditions for the map $\vb \rightarrow \tau(\vb, o)$ for a given observation $o \in \Sigma \times \Ac2$ to be contractive: $||\tau(\vb_1, o) - \tau(\vb_2, o)||_{tv} \leq \gamma ||\vb_1 - \vb_2||_{tv}$ for constant $\gamma < 1$.

\begin{definition}[Induced Matrix Norm]
Given a $n \times n$ matrix $Q$, its induced $p$-norm for $p \geq 1$ is defined as:
\[ ||Q||_p = \mathsf{sup}_{\vx \in \reals^n, \vx \not= 0}\ \frac{||Q \vx||_p }{||\vx||_p}\,.\]
Also note that for a matrix $Q$, the induced $L_1$-norm $||Q||_1$ is defined as 
\[ ||Q||_1 = \max_{j=1}^n \sum_{i=1}^n |A_{i,j}|\,, \]
the maximum over all the sum of absolute values of entries along each column of the matrix (Cf.~\cite{Weisstein/Matrix} for further details).
\end{definition}

\begin{lemma}\label{lemma:useful-lemma-1}
For any belief vectors $\vb_1, \vb_2 \in \Beliefs_n$, we have 
\[ ||T^t \vb_1 - T^t \vb_2 ||_{tv} \leq (1 - n t^*) ||\vb_1 - \vb_2||_{tv} \,.\]
\end{lemma}
\begin{proof}
Let $\oneMat_{n \times n} $ be the $n\times n$ matrix with all $1$ entries and $\oneMat_n$ be the $n\times 1$ vector of all $1$s. Let $Q = T^t - t^* \oneMat_{n \times n}$. Note that $||Q||_1$ is the maximum among the column sums of $Q$. Each column of $Q$ corresponds to a row of $T$ with $t^*$ subtracted from each entry. Therefore, each column of $Q$ sums to $1 - n t^*$.

We can write $T^t \vb = (T^t - t^* \oneMat_{n\times n}) \vb + t^* \oneMat_n \vb = Q \vb + t^* \oneMat_n$. Thus,
\[ \begin{array}{rl}
||T^t \vb_1 - T^t \vb_2 ||_{tv} & \leq || Q \vb_1 - Q \vb_2 + \cancel{t^* \oneMat_n} - \cancel{t^* \oneMat_n} ||_{tv} \\
& \leq ||Q||_1 ||\vb_1 - \vb_2||_{tv} \\ 
& \leq (1 - n t^*) ||\vb_1 - \vb_2||_{tv} \\ 
\end{array}\]
\end{proof}

Let $\alpha_j(o)$ denote $\pi_j(o)$, $\alpha_{\max}(o)  = \max_{j=1}^n \alpha_j(o)$ and $\alpha_{sum}(o) = \sum_{j=1}^n \alpha_j(o)$.
If the observation $o$ is clear from the context, we will simply write $\alpha_{\max}$ and $\alpha_{sum}$ to denote $\alpha_{\max}(o)$ and $\alpha_{sum}(o)$, respectively.

Let $\scr{D}_n = \{ \vb \in \Beliefs_n\ |\ b_j \geq t^*, \forall\ j \in [n] \}$. Following lemma~\ref{lemma:belief-lower-bounds}, we can restrict our attention to just those belief vectors in $\scr{D}_n$ since every belief state obtained through a non-zero probability sequence of observations will belong to $\scr{D}_n$.

\begin{lemma}\label{lemma:tau-contraction-bounds}
For a non-zero probability observation $o \in \Sigma \times \Ac2$ and belief states $\vb_1, \vb_2 \in \scr{D}_n$, we have
\[ ||\tau(\vb_1, o) - \tau(\vb_2, o) ||_{tv} \leq \frac{(1- nt^*) \alpha_{\max}(o)}{t^* \alpha_{sum}(o)} ||\vb_1 - \vb_2 ||_{tv} \,.\]
\end{lemma}
\begin{proof}
We have 
\[ \begin{array}{rll}
||\tau(\vb_1, o) - \tau(\vb_2, o) ||_{tv} & =  || T^t \cond(\vb_1,o) - T^t \cond(\vb_2, o) ||_{tv}  \\
& \leq (1-n t^*) || \cond(\vb_1,o) - \cond(\vb_2, o) ||_{tv} & \text{applying Lemma~\ref{lemma:useful-lemma-1}} \\ 
& \leq (1 - nt^*) \sum_{j=1}^n \left| \frac{b_{1,j} \alpha_j}{\sum_{i=1}^n b_{1, i} \alpha_i} - \frac{b_{2,j} \alpha_j}{\sum_{i=1}^n b_{2, i} \alpha_i} \right| \\ 
\end{array}\]

Note that $\sum_{i=1}^n b_{1,i} \alpha_i \geq t^* \sum_{i=1}^n \alpha_i = t^* \alpha_{sum}$ since each entry of $\vb_1$ is at least $t^*$. Similarly, we note that $\sum_{i=1}^n b_{2,i} \alpha_i \geq t^* \alpha_{sum} $. Therefore, 
\[ \begin{array}{rll}
||\tau(\vb_1, o) - \tau(\vb_2, o) ||_{tv} & \leq \frac{(1 - nt^*)}{t^* \alpha_{sum}} \sum_{j=1}^n | \alpha_j b_{1,j} - \alpha_j b_{2,j} | \\ 
& \leq \frac{1 - nt^*}{t^* \alpha_{sum}} \alpha_{\max} \sum_{j=1}^n | b_{1,j} - b_{2,j} | \\ 
& \leq  \frac{(1- nt^*) \alpha_{\max}(o)}{t^* \alpha_{sum}(o)} ||\vb_1 - \vb_2 ||_{tv} \\ 
\end{array}\]
\end{proof}
Let us define $\kappa(o) = \frac{\alpha_{\max}}{\alpha_{sum} + n \alpha_{\max}}$.
\begin{lemma}
The map $\vb \mapsto \tau(\vb, o)$ is contractive if $t^* > \kappa(o)$.
\end{lemma}
\begin{proof}
Using Lemma~\ref{lemma:tau-contraction-bounds}, we note that $\vb \mapsto \tau(\vb, o)$ is contractive if $\frac{(1- nt^*) \alpha_{\max}(o)}{t^* \alpha_{sum}(o)} < 1$. 
\[ \begin{array}{ll}
\frac{(1- nt^*) \alpha_{\max}(o)}{t^* \alpha_{sum}(o)} < 1 \\ 
\ \Leftrightarrow\ (1 - nt^*) \alpha_{\max} < t^* \alpha_{sum}  & \because\ t^* \alpha_{sum} > 0 \\ 
\ \Leftrightarrow\ t^* (\alpha_{sum} + n \alpha_{\max}) > \alpha_{\max} & \ \text{rearranging terms} \\ 
\ \Leftrightarrow\ t^* > \frac{\alpha_{\max}}{\alpha_{sum} + n \alpha_{\max}} \,.
\end{array}\]
\end{proof}

Having established these results, we proceed with the proof of Theorem~\ref{thm:completeness-ism-algo}. Let us assume that $t^* > \kappa_{\max} \geq \kappa(o)$ for all non-zero probability observations $o$. 
 
First, we conclude that Algorithm~\ref{alg:info-state-machine-search} will never return \textsf{FAIL} (line~\ref{consistency-check-2}). This is because, any edge $m \xrightarrow{o} m'$ wherein $\vb(m') = \tau(\vb(m), o)$ will be consistent due to the contractivity of $\tau$. In other words,
for any $\vb \in \scr{D}_n$ such that $||\vb - \vb(m)||_{tv} \leq \lambda$, we have
\[ 
\begin{array}{rl}
||\tau(\vb, o) - \vb(m')||_{tv}  & \leq \frac{(1- nt^*) \alpha_{\max}(o)}{t^* \alpha_{sum}(o)} ||\vb - \vb(m) ||_{tv}  \\ 
& \leq \gamma(o) \lambda \\ 
\end{array} \] 
wherein $\gamma(o) = \frac{(1- nt^*) \alpha_{\max}(o)}{t^* \alpha_{sum}(o)}  < 1$. 
Let $L^* = \max_{o \in O} \gamma(o)$. Clearly, $L^* < 1$, as well.

For a given $\delta > 0$, let $\scr{B}_{\delta}(\vb) = \{ \widehat{\vb} \in \scr{D}_{n} \ |\ ||\vb - \widehat{\vb} ||_{tv} \leq \delta \}$ be a ball of size $\delta$ in the total-variation norm over belief states.
Next consider a ``packing'' of the belief space $\scr{D}_n$.
\begin{definition}[Minimal Packing with balls of size $\delta$]
A minimal packing of $\scr{D}_n$ using balls of size $\delta$ is a family of $N$ sets $\scr{F} =  \left\{ \scr{B}(\vb_i, \delta) \right\}  $ for $i \in [N]$ that (a) covers the entire belief space $ \bigcup_{S \in \scr{F}} S \supseteq \scr{D}_n$ and (b) minimizes the size of the family $N$ over all such covers.
\end{definition}

Let $\scr{F}$ be a minimal packing of the belief space $\scr{D}_n$ with balls of size $(1 - L^*) \lambda$. By the compactness of $\scr{D}_n$, we note that $\scr{F}$ is finite. We now prove that for the automaton constructed by Algorithm~\ref{alg:info-state-machine-search}, we cannot have two states $m, m'$ such that $\vb(m), \vb(m')$ belong to the same ball in $\scr{F}$. 

\begin{lemma}\label{lemma:useful-lemma}
Let $\scr{F}$ represent a family of sets that form a minimal $\delta = (1-L^*) \lambda$ packing of the belief space  $\scr{D}_n$. Algorithm~\ref{alg:info-state-machine-search} during its run cannot produce two states $m, m' \in M$ such that $\vb(m) \in S$ and $\vb(m') \in S$ for $S \in \scr{F}$.
\end{lemma}
\begin{proof}
We will prove this by contradiction. Let $m, m'$ be two states created such that $||\vb(m) - \vb(m') ||_{tv} \leq (1 - L^*) \lambda$. In fact, let us assume that $(m, m')$ are the very first pair of states constructed during the execution of Algorithm~\ref{alg:info-state-machine-search} with this property. 

Let us assume that $m$ is the first state constructed, followed by $m'$.  Let  $\vb' = \vb(m')$. We create the state $m'$ because of we have $\vb' = \tau(\vb(m_1), o)$ for some previously added state $m_1$ and observation $o \in \Sigma'$ (line~\ref{next-belief-state} of Algorithm~\ref{alg:info-state-machine-search}). 

The call to \textsf{findClosestState} (line~\ref{find-closest-state}) must return the state $m$ since if it returned some other state $m_2$ then $||\vb(m_2) - \vb(m')||_{tv} \leq ||\vb(m) - \vb(m')||_{tv} \leq (1-L^*) \lambda$. This means that $(m, m_2)$ are a pair of already created states that contradicts the statement of this theorem. However, this goes against our assumption that $(m, m')$ is the very first pair created. Therefore, $m_2 = m$.

By assumption, 
\[ ||\vb(m) - \vb(m') ||_{tv} \leq (1 - L^*) \lambda \,.\]
Also, since $\vb(m') = \tau(\vb(m_1), o)$ and $\tau$ is contractive, we know that any belief state $\vb \in  \scr{D}_{n}$ such that $||\vb - \vb(m_1)||_{tv} \leq \lambda$, 
\[ ||\tau(\vb, o) -  \vb(m')||_{tv} \leq L^* ||\vb - \vb(m_1)||_{tv} \leq L^* \lambda \,.\]

Therefore, 
\begin{align*}
    ||\tau(\vb, o) - \vb(m) ||_{tv} & \leq ||\tau(\vb, o) - \vb(m')||_{tv} + || \vb(m') - \vb(m) ||_{tv} \\
    & \leq L^* \lambda + (1 - L^*) \lambda \\ 
    & \leq \lambda \\ 
\end{align*}

We thus know that the edge from $m_1 $ to $m$ will be consistent. Therefore, the state $m'$ is never created by our algorithm because the then-branch of the condition in line~\ref{nearby-state-found} in Algorithm~\ref{alg:info-state-machine-search} is executed, yielding a contradiction.
 \end{proof}

As a result, we have proven a finite upper bound on the number of possible states Algorithm~\ref{alg:info-state-machine-search} can produce which happens to be the \emph{packing number} of the minimum cardinality family of balls of radius $(1-L^*) \lambda$ in the total variation norm that covers the belief space $\scr{D}_n$.

This concludes the proof of Theorem~\ref{thm:completeness-ism-algo}.
\else 
The full proof is provided in the extended version. 
\fi 

\begin{example}
For all observations $o$ in the RPS example from Figure~\ref{fig:rps-rewards-and-policies}, $\alpha_{\max}(o) = 0.5$, $\alpha_{sum}(o) = \frac{4}{3}$. We have $\kappa_{\max} = \kappa(o) = \frac{0.5}{4/3 + 4 (0.5)} = \frac{3}{20} = 0.15 $.   Theorem~\ref{thm:completeness-ism-algo}, guarantees for any matrix $T$ all of whose entries exceed $0.15$ is guaranteed to yield a finite state ISM for any $\lambda > 0$. Interestingly, the matrix in Figure~\ref{fig:rps-rewards-and-policies} \emph{does not} satisfy this condition and nevertheless yields finite ISM for $\lambda = 0.25$ ( Figure~\ref{fig:consistent-ism-rps-example}).
 \end{example}

\noindent\textbf{Robustness:} Suppose we designed an ISM $\scr{M}$ that is consistent for $\lambda > 0$ assuming matrix $T = T_{D}$, whereas in reality $\player2$ switches policies according to $T = T_{A}$, wherein $T_{A} \not= T_{D}$. We will prove that the ISM $\scr{M}$ which is consistent for $T= T_{D}$ and $\lambda > 0$ will remain consistent for $T = T_{A}$ for a different value $\lambda = \overline{\lambda}$.
Let $t^*_d = \min_{i, j \in [n]} T_{D,i,j}$ and $t_a^* = \min_{i, j \in [n]} T_{A,i,j}$ be the minimum entries in the matrices $T_D$ and $T_A$ respectively. 
Let us define the function 
\[ L(T_A, T_D, \scr{G}, \Pi) =  \max_{o \in O} \frac{ (1- n\  \max(t_a^*, t_d^*)) \alpha_{\max}(o)}{\min(t_a^*, t_d^*) \alpha_{sum}(o)}\]
\begin{theorem}\label{thm:robustness-main}
If $t_a^* > 0, t_d^* >0$ and $L(T_A, T_D, \scr{G}, \Pi) < 1$ then the ISM $\scr{M}$ is consistent under the matrix $T_D$ with the consistency parameter
$ \overline{\lambda} = \frac{\lambda + ||(T_A - T_D)^t||_1}{1 - L(T_A, T_D, \scr{G}, \Pi) } $.
\end{theorem}

$||T||_1$ refers to the induced $1-$norm of matrix $T$~\cite{Weisstein/Matrix}. 
\ifextendedversion
\begin{proof}

Let us assume that $T_{A}$ is the actual matrix used by $\player2$ whereas $T_{D}$ is the matrix assumed during the design of the consistent ISM $\scr{M}$. Let each entry of $T_{A}$ be at least $t_{a}^* $ whereas $t_d^* $ is the minimal entry in the matrix $T_D$. We assume that $t_a^* > 0$ and $t_d^* > 0$. 

For a belief state $\vb \in \Beliefs_n$, let $\tau_{D}(\vb, o)$ denote the updated belief state using the design assumption $T_D$:
\[ \tau_D(\vb, o) = T_D^t \times \cond(\vb, o) \,.\]
Likewise, let $\tau_{A}$ be the updated belief state using the actual play matrix $T_A$:
\[ \tau_A(\vb,o) = T_A^t \times \cond(\vb, o) \,.\]
Let $t_{\max} = \max(t_a^* , t_d^*)$ and $t_{\min} = \min(t_a^*, t_b^*)$. Recall the definitions: $\alpha_j = \pi(o)$, $\alpha_{\max}(o) = \max_{j=1}^n \alpha_j$ and $\alpha_{sum}(o) = \sum_{j=1}^n \alpha_j$.

\begin{lemma}
Let $\vb_1, \vb_2$ be two belief states such that  for all $j \in [n]$, 
$b_{1,j} \geq t_a^*$ and $b_{2,j} \geq t_d^*$;
and $o$ be a non-zero probability observation. 
\[ ||\tau_A(\vb_1, o) - \tau_D(\vb_2, o)||_{tv} \leq ||(T_A - T_D)^t||_1 + \frac{(1 - n t_{\max}) \alpha_{\max}(o)}{t_{\min} \alpha_{sum}(o)} || b_1 - b_2 ||_{tv} \,. \]
\end{lemma}
\begin{proof}
\[\begin{array}{ll}
||\tau_A(\vb_1, o) - \tau_D(\vb_2, o)||_{tv} \\
\;\;\; = || T_A^t \cond(\vb_1, o) - T_D^t \cond(\vb_2, o) ||_{tv} & \text{(* let }\ \vb' := \cond(\vb, o) \text{*)}\\ 
\;\;\;  = ||T_A^t (\vb_1' - \vb_2') ||_{tv} + || (T_A - T_D)^t \vb_2'||_{tv} \\ 
\;\;\; \leq (1 - nt_a^*) || \vb_1' - \vb_2'||_{tv} + ||(T_A - T_D)^t||_1 \times \underset{=1}{\underbrace{||\vb_2'||_1}} & \text{(*Cf. Lemma~\ref{lemma:useful-lemma-1}*)}\\ 
\end{array}\]

Consider another derivation that proceeds as follows:
\[\begin{array}{ll}
||\tau_A(\vb_1, o) - \tau_D(\vb_2, o)||_{tv} \\
\;\;\; = || T_A^t \cond(\vb_1, o) - T_D^t \cond(\vb_2, o) ||_{tv} & \text{(* let }\ \vb' := \cond(\vb, o) \text{*)}\\ 
\;\;\;  = ||T_D^t (\vb_1' - \vb_2') ||_{tv} + || (T_A - T_D)^t \vb_1'||_{tv} \\ 
\;\;\; \leq (1 - nt_d^*) || \vb_1' - \vb_2'||_{tv} + ||(T_A - T_D)^t||_1 \times \underset{=1}{\underbrace{||\vb_1'||_1}} & \text{(*Cf. Lemma~\ref{lemma:useful-lemma-1}*)}\\ 
\end{array}\]

Combining, we obtain:
\[  ||\tau_A(\vb_1, o) - \tau_D(\vb_2, o)||_{tv}  \leq \underset{ = 1 - n t_{\max} }{\underbrace{\min( 1- n t_a^*, 1 - n t_d^*)}} || \vb_1' - \vb_2'||_{tv} + ||(T_A - T_D)^t||_1 \,.\]

We will now calculate bounds on $||\vb_1' - \vb_2'||_{tv}$.
\[ \begin{array}{ll}
||\vb_1' - \vb_2'||_{tv} \;\;  = \;\; \sum_{j=1}^n \left| \frac{\alpha_j b_{1,j}}{\sum_{i=1}^n \alpha_i b_{1, i}} - \frac{\alpha_j b_{2,j}}{\sum_{i=1}^n \alpha_i b_{2, i}} \right| \\
\end{array}\]

Note that $\sum_{i} \alpha_i b_{1,i} \geq t_a^* \alpha_{sum} \geq t_{\min} \alpha_{sum}$ since $b_{1,i} \geq t_{a}^*$. Likewise, $\sum_i \alpha_i b_{2,i} \geq t_{\min} \alpha_{sum} $. Therefore,
\[ \begin{array}{ll}
||\vb_1' - \vb_2'||_{tv} \;\;  = \;\; \sum_{j=1}^n \left| \frac{\alpha_j b_{1,j}}{\sum_{i=1}^n \alpha_i b_{1, i}} - \frac{\alpha_j b_{2,j}}{\sum_{i=1}^n \alpha_i b_{2, i}} \right| \\
\;\;\;\;\; \leq \frac{1}{t_{\min}\alpha_{sum}} \sum_{j=1}^n | \alpha_j b_{1,j} - \alpha_j b_{2,j} | \leq \frac{\alpha_{\max}}{t_{\min} \alpha_{sum}} ||\vb_1 - \vb_2 ||_{tv} \\ 
\end{array}\]

Combining, we obtain,
\[\begin{array}{ll}
||\tau_A(\vb_1, o) - \tau_D(\vb_2, o)||_{tv} \\
\;\;\; \leq (1 - nt_{\max}) || \vb_1' - \vb_2'||_{tv} + ||(T_A - T_D)^t||_1 \\ 
\;\;\;\;\;\; \leq \frac{(1 - nt_{\max})\alpha_{\max}}{t_{\min} \alpha_{sum}} ||\vb_1 - \vb_2||_{tv} + ||(T_A - T_D)^t||_1\,.
\end{array}\]
This completes the proof of this lemma.
\end{proof}

Let $\sigma_t$ be a sequence of observations with non-zero probability and $\vb = \tau_D(\vb_0, \sigma_t)$ and $\hat{\vb} = \tau_A(\vb_0, \sigma_t)$ for the uniform initial belief state $\vb_0$. Let $m$ be the state in the ISM $m = \delta(m_0, \sigma_t)$ with associated belief state $\vb(m)$.

We will now proceed to the proof of the original theorem. Let us consider an edge $m \xrightarrow{o} m'$ in the automaton. We will show that the edge is $\overline{\lambda}$ consistent under $\tau_A$. Let $\hat{\vb}$ be any belief state such that $\hat{b}_j \geq t_a^*$ and 
\[ ||\hat{\vb} - \vb(m) ||_{tv} \leq \overline{\lambda} \,,\]
wherein 
\[ \overline{\lambda} = \frac{\lambda + ||(T_A - T_D)^t||_1}{1 - L(T_A, T_D, \scr{G}, \Pi) } \,, \]
and  $L(T_A, T_D, \scr{G}, \Pi) = \max_{o \in O} \frac{ (1- nt_{\max}) \alpha_{\max}}{t_{\min} \alpha_{sum}}$ wherein
$L(T_A, T_D, \scr{G}, \Pi) < 1$ by assumption. 

We wish to prove that 
\[ ||\tau_A(\hat{\vb},o) - \vb(m') ||_{tv} \leq \overline{\lambda}\,.\]

First, we note that for any belief state $\vb$ such that $b_j \geq t_d^*$
\[ ||\tau_D(\vb, o) - \tau_A(\hat{\vb}, o)||_{tv} \leq \frac{(1 - n t_{\max}) \alpha_{\max}}{t_{\min} \alpha_{sum}} || \vb - \hat{\vb}||_{tv} + ||(T_A - T_D)^t||_1 \,. \]

Therefore, 

\[ \begin{array}{l}
|| \tau_A(\hat{\vb}, o) - \vb(m') || \leq ||\tau_A(\hat{\vb}, o) - \tau_D(\vb(m), o)|| + ||\tau_D(\vb(m), o) - \vb(m')|| \\ 
\;\;\;\;\;\;\;\;\;\; \leq \frac{(1 - n t_{\max}) \alpha_{\max}(o)}{t_{\min} \alpha_{sum}(o)} || \hat{\vb} - \vb(m)||_{tv} + ||(T_A - T_D)^t||_1  + \lambda  \\ 
\;\;\;\;\;\;\;\;\;\; \leq L(T_D, T_{A}, \Pi) \overline{\lambda} +  \underset{ = (1 - L(T_D, T_A, \scr{G}, \Pi))\overline{\lambda}}{\underbrace{||(T_A - T_D)^t||_1  + \lambda }}\\ 
\;\;\;\;\;\;\;\;\;\; \leq \overline{\lambda}  \\ 
\end{array}\]

This completes the proof. 
\end{proof}
\else
Proof is provided in the extended version.
\fi

\section{Experimental Evaluation}\label{sec:experimental-evaluation}
%\input{experiments.tex}
%\vspace{-1em}
We present an experimental evaluation based on an implementation of the ideas mentioned thus far. Our implementation uses the Python programming language and inputs a user-defined game structure, $n$ policies for player $\player2$, values for parameters $\lambda > 0$. For each case, the policy design Markov chain whose transition system is given by $T(\epsilon)$,
such that  $T(\epsilon)_{i,i} = \epsilon$ and $T(\epsilon)_{i,j} = \frac{\epsilon}{n-1}$ when $i \not= j$. In other words, player plays the same policy as previous step with probability $\epsilon$ and switches to a different policy uniformly with probability $(1-\epsilon)/(n-1)$.   Our implementation uses the Gurobi optimization solver \cite{gurobi} to implement the consistency checks described in Section~\ref{sec:checking-consistency} and uses it to implement the consistent information state machine synthesis as described in Section~\ref{sec:consistent-ism-synthesis}. 
%The construction of the MDP and a textbook policy iteration algorithm have been implemented to find an optimal policy for player $\player1$. Our implementation simulates the policy for $\player1$ against a choice of possible policies for $\player2$. 

\paragraph{Performance Evaluation on Benchmark Problems.}
We consider benchmarks for evaluating our approaches in terms of the ability to construct finite information state machines, the sizes of these machines and the performance of the resulting policies synthesized by our approach.
\begin{compactenum}
    \item \textsc{rps}: The rock-paper-scissors game and $\player2$ policies as described in Example~\ref{ex:rps-example}. 
    \item \textsc{rps-mem}: The rock-paper-scissors game but with ``memory'' of the previous move by each player. This game has $9$ states that remember the previous move of each player, and the policies for $\player2$ model behaviors such as ``play action now that would  have beaten $\player1$ in the previous turn'' or ``repeat the previous action of $\player1$''. 
    %or ``play action that $\player1$ played in previous turn.''.
   \item \textsc{Anticipate-n-Avoid}(N) consists of a circular corridor with $N$ rooms numbered $1, \ldots, N$ with four designated rooms marked as meeting zones. $\player2$ chooses one of four policies that navigate them to one of the meeting rooms whereas the rewards for $\player1$ are negative if they happen to be in the same cell as $\player2$ or in an adjacent cell while the rewards are positive if they happen to be farther away. The game has $N^2$ states for $N$ rooms. 
\end{compactenum}
The game structures and the policies for $\player2$ are given in the appendix. 

\begin{table}[t]
\begin{center}
{\footnotesize 
\begin{tabular}{|c|c||ccccc||ccccc||}
\hline 
Benchmark & Size &  \multicolumn{5}{c||}{$\lambda=0.1$} & \multicolumn{5}{c||}{$\lambda=0.05$} \\
\cline{3-12}
& & $\epsilon$  & $|M|$ & $T_{alg1}$ & $|\text{MDP}|$ & $T_{PI}$&  $\epsilon$ & $|M|$ & $T_{alg1}$ & $|\text{MDP}|$ & $T_{PI}$\\
\hline 
\textsc{rps} & (1, 3, 3, 4) &  $0.5$ & 6 & 0.34 &  6 & $<0.01$ & 0.5 & 10 & 0.4 & 10 & $< 0.01$ \\ 
 & &   $0.4$ & 20 & 0.92 & 20 & $< 0.01$ &   0.4 & 29 & 1.1 & 29 & $< 0.01$\\ 
 & &   $0.3$ & 80 & 4.6 & 80 & $0.01$ &   0.3 & 115 & 4.9 & 115 & $0.02$\\ 
 & &  $0.2$ & $\times$ & 0.02 & \multicolumn{2}{c ||}{- Alg. 1 Fail -} &   0.2 & $\times$ & 0.4 & \multicolumn{2}{c||}{- Alg. 1 Fail -}\\[5pt] 
\textsc{rps-mem} & (9, 3, 3, 9) &  0.6 & 77 & 28.7 & 244 & 0.1 & 0.6 & 176 & 55 & 526 & 0.1  \\ 
 &  & $0.55$ & 228 & 117 & 688 & 1.3 & 0.55 & 448 & 215 & 1342 & 0.34\\ 
 & &  $0.5$ & 834 & 743 & 2500 & 4.6 &  0.5 & 1516 & 1101 & 4546 & 1.8\\
 & &  $0.45$ & $\times$ & 14.2 & \multicolumn{2}{c |}{- Alg. 1 Fail - } &   0.45 & \multicolumn{4}{c||}{ - Timeout $> 1hr$- }\\[5pt]
 \textsc{ant.-avd.} & (625, 3, 3, 4) &  $0.55$ & 7 & 1.5 & 1701 &  1.8 & 0.55 & 17 & 2.6 & 3526 & 4.2 \\
 & & $0.5$ & 4.3 & 12 & 2726 & 3.2 &   0.5 & 28 & 6.9 & 5226 & 12.9 \\ 
 & &  $0.45 $ & 8.8 & 26 & 4926 & 11.5  &  0.45 & 61 & 14.5 & 8326 & 19.5\\ 
  & &  $0.4$ & 19.7 & 66 & 10042 & 26.5 &  0.4 & 137 & 34.6 & 16882 & 50.5\\ 
  & &  $0.35$ & 68.5 &  194 & 24592 &  84  &  0.35 & 366 & 77.6 & 37770 & 112.1\\ 
  & &  $0.3$ & $\times$ & 4 &   \multicolumn{2}{c |}{- Alg. 1 Fail - } &   0.3 & 1289 & 305.4 & 126395 & 431.1 \\[5pt]
%\textsc{Ant-Act} & (2080,6, 18, 6) & 0.6 & 13 & 496 & 7904& 17.5 &  0.6 & 26 & 767.8 & 14594 & 9.7\\ 
% & &  0.55 & 25 & 804 & 13872 & 30 &  0.55 & 53 & 1307 & 22373 & 73.3\\
% & & 0.5 & 50 & 1519 & 23006 & 41.2 &  0.5 & 108 & 2431.6 & 41222 & 73.3\\
% & & 0.45 & $\times$ & 0.7 &  \multicolumn{2}{c |}{- Alg. 1 Fail - } & 0.45 & \multicolumn{4}{c||}{-Timeout $> 1 hr$-}\\
 \hline 
\end{tabular}
}
\end{center}
\caption{Performance results of our approach on various benchmarks and different values of the parameters $\lambda, \epsilon$. ``Size'' is a four-tuple consisting of $(|S|, |\Ac1|, |\Ac2|, |\Pi|)$, $T_{alg1}$ is time taken (seconds) to run Algorithm~\ref{alg:info-state-machine-search} and $T_{PI}$ is time taken (seconds)  for policy iteration to converge (discount factor $\gamma = 0.95$).  Experiments were run on  Linux server with four 2.4 GHz Intel Xeon CPUs and $64GB$ RAM. }
\vspace{-2em}
\label{tab:benchmark-performance}
\end{table}

Table~\ref{tab:benchmark-performance} shows the performance over these benchmarks. We have four benchmarks as described briefly above and in detail in the Appendices~\ref{Appendix:rps-mem}, and \ref{Appendix:ant-avoid}. For these benchmarks the number of states ranges from $1$ for the rock-paper-scissors game to $2080$ states for the \textsc{Anticipate-Action} game. Similarly, the number of actions of each player and the number of policies employed by $\player2$ are reported. For each game, we choose various values of $(\lambda, T(\epsilon))$ and report the overall performance in terms of number of states of the information state machine, the time taken to construct it, the size of the MDP and the time taken to compute an optimal policy using policy iteration. Since the transition matrix $T = T(\epsilon)$, we note that $\min (T_{i,j}) = \frac{\epsilon}{n-1}$ provided $\epsilon \leq \frac{n-1}{n}$. We ran two series of experiments for each benchmark by fixing $\lambda$ and decreasing $\epsilon$ for the matrix $T(\epsilon)$ until Algorithm~\ref{alg:info-state-machine-search} reports a failure or times out after one hour. The first observation is that our approach works for values of $t^* = \frac{\epsilon}{n-1}$ that are smaller than the limit suggested by Theorem~\ref{thm:completeness-ism-algo}. At the same time, we note that as $\epsilon$ decreases, the size of the automaton $\scr{M}$ and the corresponding size of the MDP obtained by composing the automaton with the game all increase, as does the time taken to construct. Also, if Algorithm~\ref{alg:info-state-machine-search} fails, it happens very quickly, allowing us to increase $\epsilon$ until we succeed.  
%Overall, for most cases our approach constructs a consistent information state machine. 

\newcommand\commentout[1]{}
\commentout{
\begin{figure}[t]
\begin{center}
\begin{tabular}{ccc}
\includegraphics[width=0.3\textwidth]{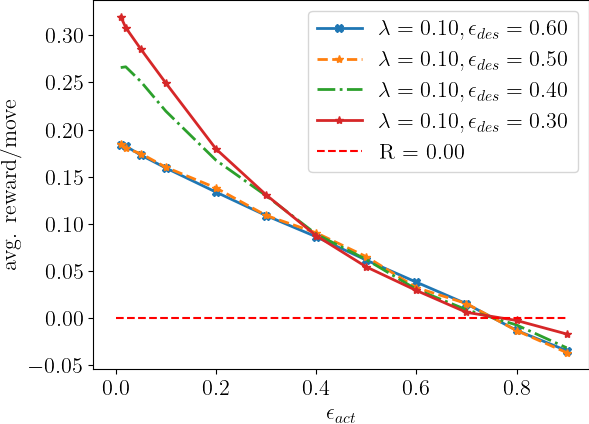} & \includegraphics[width=0.3\textwidth]{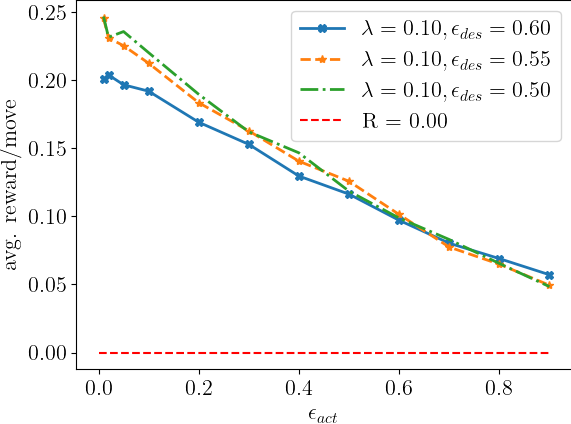} & \includegraphics[width=0.3\textwidth]{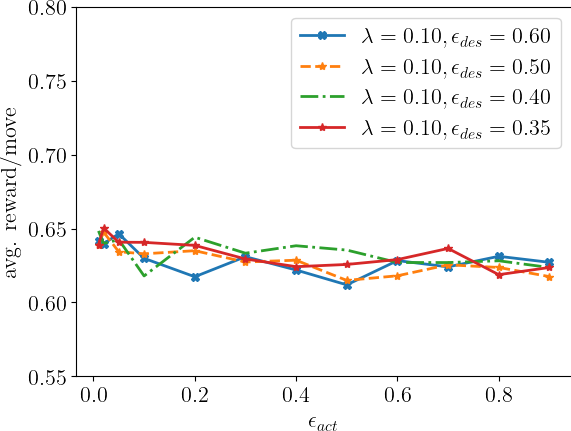}  \\[5pt] 
\includegraphics[width=0.3\textwidth]{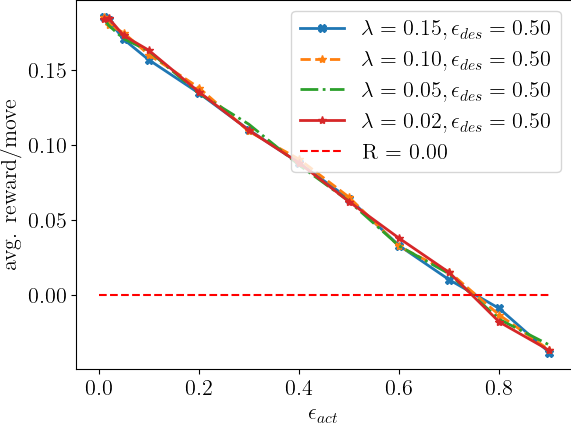} & \includegraphics[width=0.3\textwidth]{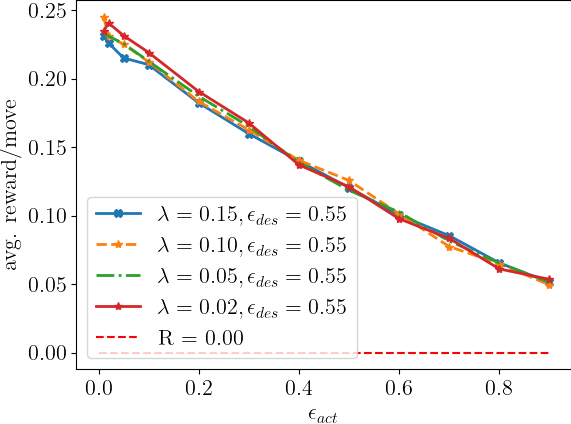} & \includegraphics[width=0.3\textwidth]{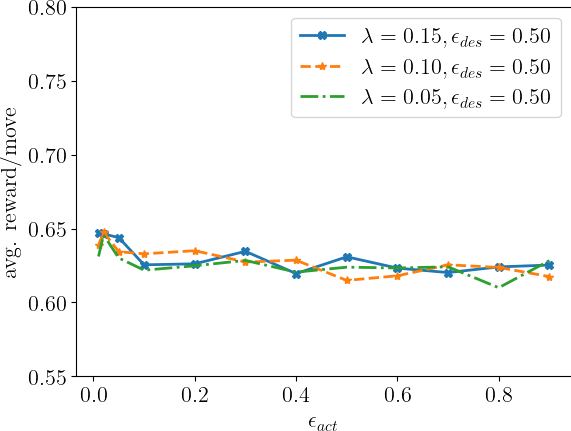} \\
\textsc{rps} & \textsc{rps-mem} & \textsc{Ant.-Avoid}  \\ 
\end{tabular}
\end{center}
\vspace*{-0.5cm}
\caption{Average reward/move (y-axis) plotted against varying values of $\epsilon_{act}$ (x-axis) for the benchmarks. Top row fixes $\lambda = 0.1$ while varying $\epsilon$ whereas bottom row fixes $\epsilon$ while varying $\lambda$.}\label{fig:avg-reward-per-episode}
\end{figure}

\vspace{0.4em}\noindent\textbf{Performance of the Policy.} Next, we simulate the optimal policy designed by our approach for a given value of $\lambda, T_D=T(\epsilon_{des})$, against policies for player $\player2$ who switches with probability $T_A = T(\epsilon_{act})$. Figure~\ref{fig:avg-reward-per-episode} shows the reward per move for the policy computed by our approach simulated against various values of $T(\epsilon_{act})$ along the x-axis. For the rock-paper-scissors games, we note that the best that one can do without knowledge of the opponent moves is to obtain an average reward of $0$. In this sense, our policy does well, especially when $\epsilon_{act}$ is small, i.e, $\player2$ sticks to the same policy for longer. At the same time, decreasing $\epsilon_{des}$ yields better reward performance for small values of $\epsilon_{act}$ but this performance degrades quickly whereas a larger $\epsilon_{des}$ value yields lower performance that does not degrade as quickly, due to the conservative assumption about the opponent.
% while designing the policies.

\begin{figure}[t]
\begin{center}
\begin{tabular}{ccc}
\includegraphics[width=0.3\textwidth]{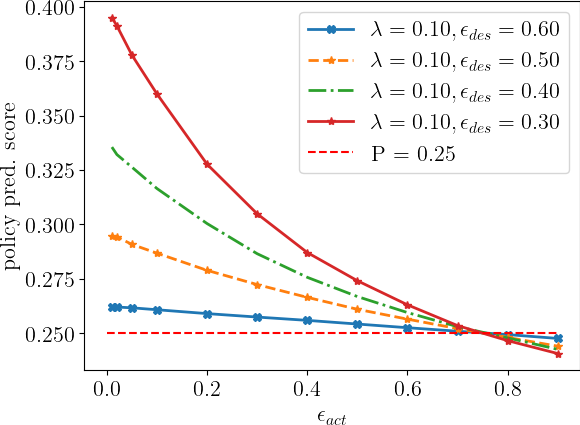} & \includegraphics[width=0.3\textwidth]{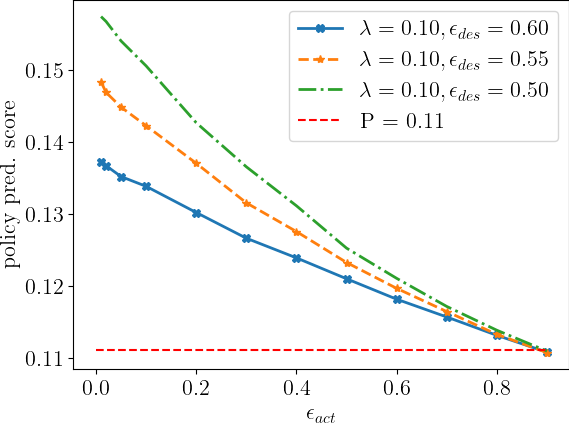} & \includegraphics[width=0.3\textwidth]{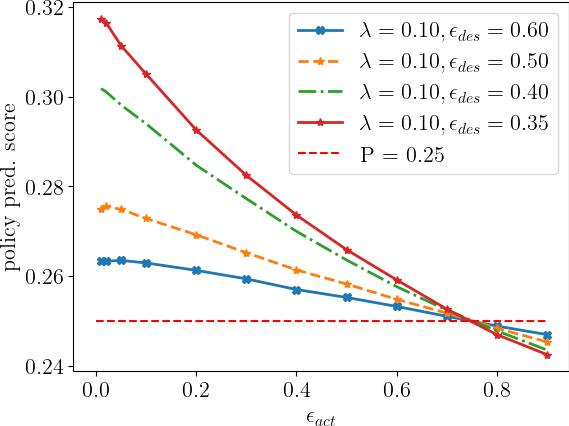}  \\
\textsc{rps} & \textsc{rps-mem} & \textsc{Ant.-Avoid} \\ 
\end{tabular}
\end{center}
\vspace*{-0.5cm}
\caption{ Policy prediction score: avg. probability per move that the ISM's belief state assigns to the actual policy of $\player2$ against varying values of $\epsilon_{actual}$. All plots have $\lambda {=}0.1$ and varying $\epsilon_{des}$.}
\vspace{-2em}
\label{fig:avg-policy-pred-per-episode}
\end{figure}

Figure~\ref{fig:avg-policy-pred-per-episode} plots the average policy prediction score, i.e, the average over the probability  assigned by the belief state tracked by the ISM to  the actual policy employed by $\player2$ at each step. 
This is an important metric since it measures how well our information state machine anticipates the policies employed by $\player2$. 
Note that for small values of $\epsilon_{act}$, the policy prediction scores are higher as expected since $\player2$ sticks to a given policy for longer time periods. 
However, as $\epsilon_{act}$ increases, there is a point where the policy prediction does worse than simply picking a policy at random which has a $\frac{1}{|\Pi|}$ chance of succeeding. 
The results show that our approach anticipates the policy of $\player2$ through the approximate belief state. The good reward performance in Figure~\ref{fig:avg-reward-per-episode} follows the same trend as the policy prediction in Figure~\ref{fig:avg-policy-pred-per-episode}.
}

\begin{table}[t]
\begin{center}
\begin{tabular}{||c | cc|cc||c|cc|cc||}
\hline 
\multicolumn{5}{||c||}{Ikea-Shelf-Drawer ($|\scr{G}| = 18, |\Pi|=7$)} & \multicolumn{5}{c||}{Ikea-TV-Bench ($|\scr{G}| = 18, |\Pi|=13$)}\\
\rowcolor{gray!30}$(\lambda, \epsilon)$ & $|\scr{M}|$ & $T_M$ &  $r_{\mbox{avg}}$ & $\mbox{ap}_{\mbox{avg}}$& 
 $(\lambda, \epsilon)$ & $|\scr{M}|$ & $T_M$ &  $r_{\mbox{avg}}$ & $\mbox{ap}_{\mbox{avg}}$\\
\hline 
 (0.01, 0.5) & 344 & 97.7 & 0.137 & 0.407 & 
 (0.01, 0.5) & 846 & 245.5 & 0.203 & 0.389 \\ 
%% Drawer
  (0.01, 0.6) & 917 & 291 & 0.137 & 0.418 &
%% TV-Bench
  (0.01, 0.6) & 2852 & 1006 & 0.21 & 0.404 \\ 
%% Drawer
  (0.01, 0.7) & 3547 & 1324.5 & 0.137 & 0.43 & 
  %% TV-Bench
  (0.01, 0.7) & \multicolumn{4}{c||}{- timeout $> 3600s$}\\
 %% Drawer
  (0.02, 0.5) & 189 & 65.4 & 0.137 & 0.407 & 
  %% TV-Bench
  (0.02, 0.5) & 425 & 142.8 & 0.198 & 0.389 \\ 
 %% Drawer
  (0.02, 0.6) & 472 & 168.5 & 0.137 & 0.418 & 
  %% TV-Bench
  (0.02, 0.6) & 1263 & 486.32 & 0.21 & 0.404 \\ 
 %% Drawer
  (0.02, 0.7) & 1566 & 615.2 & 0.137 & 0.43 & 
  (0.02, 0.7) & \multicolumn{4}{c||}{- Algo.~\ref{alg:info-state-machine-search} fail -}\\
\hline 
\multicolumn{5}{||c||}{Ikea-Coffee-Table ($|\scr{G}| = 15, |\Pi|=12$)} & \multicolumn{5}{c||}{Cataract-Surgery ($|\scr{G}| = 36, |\Pi| = 14 $)} \\
\rowcolor{gray!30}$(\lambda, \epsilon)$ & $|\scr{M}|$ & $T_M$ &  $r_{\mbox{avg}}$ & $\mbox{ap}_{\mbox{avg}}$& 
 $(\lambda, \epsilon)$ & $|\scr{M}|$ & $T_M$ &  $r_{\mbox{avg}}$ & $\mbox{ap}_{\mbox{avg}}$\\
\hline 
%% Coffee Table
(0.01, 0.5) & 521 & 150 & 0.181 & 0.408 & 
%% Cataract
(0.01, 0.4) & 399 & 236.8 & 0.287 & 0.512 \\ 
%% Coffee Table
(0.01, 0.6) & 1441 & 494 & 0.181 & 0.420 & 
%% Cataract 
(0.01, 0.5) &1360 & 846.7 & 0.287 & 0.518 \\
%% Coffee Table
(0.01, 0.7) & \multicolumn{4}{c||}{ - timeout $> 3600s$} &
(0.01, 0.6) & \multicolumn{4}{c||}{ - timeout $> 3600s$} \\
%% Coffee Table 
(0.02, 0.5) & 279 &  115 & 0.18 & 0.409 & 
%% Cataract
(0.02, 0.4) & 207 & 138 & 0.287 & 0.512 \\ 
%% Coffee Table 
(0.02, 0.6) & 705 & 292 & 0.178 & 0.420 & 
%% Cataract
(0.02, 0.5) & 626 & 371 & 0.287 & 0.518 \\
%% Coffee Table 
(0.02, 0.7) &  \multicolumn{4}{c||}{- Algo.~\ref{alg:info-state-machine-search} fail -} & 
%% Cataract
(0.02, 0.6) & 2404 & 1642  & 0.287 & 0.525 \\ 
\hline 
\end{tabular}
\end{center}
\caption{Performance data on tool prediction problem for various task sequences. $|\scr{G}|$ denotes size of automaton , $|\Pi|$: number of policies for $\player2$, $|\scr{M}|$: ISM size, 
$T_M$: time taken by Algo.~\ref{alg:info-state-machine-search}, $r_{\mbox{avg}}:$ average reward per move, $\mbox{ap}_{\mbox{avg}}:$ average probability of $\player2$'s action at each step using ISM belief state.}\label{tab:task-action-prediction-performance}
\vspace*{-0.8cm}
\end{table}

\paragraph{Next Tool Usage Prediction.}
We study the performance of our approach on two datasets involving human task performance: (a) the IKEA ASM dataset that consists of $371$ individual furniture assemblies of four distinct furniture models, wherein the actions performed by the human assembler are labeled using a neural network (CNN) to yield  sequences of actions performed by the human~\cite{Ben-Shabat+Others/2021/WACV}; and (b) the CATARACTS dataset consisting of 25 cataract surgery videos, wherein a CNN is used to identify the sequence of tools employed by the surgeon~\cite{Hassan+Others/2019/CATARACTS}. 
 
 We first used automata learning tool flexfringe to construct a DFA model from a training set consisting of 75\% of the sequences in each dataset~\cite{Verwer+Others/2017/flexfringe}. Flexfringe successfully constructed a DFA that includes the sequences of actions/tools used (Cf. Appendix~\ref{app:ikea-cataract-benchmark-details} ). The game graph $\scr{G}$ consists of the automata states and edges. The transitions between states are governed by the actions of $\player2$. The actions of $\player1$ are the same as that of $\player2$: $\Ac1 = \Ac2$. The goal of $\player1$ is to predict the next action/tool usage by $\player2$ based on knowledge of the current state. The reward $R(s, a_1, a_2) = 1$ if $a_1 = a_2$ (i.e, $\player1$'s action matches that of $\player2$) and $R(s, a_1, a_2) = -1$ otherwise.  The policies of $\player2$ are also constructed from the training data as well. 
 For each sequence $\sigma$ in the training data we collect the set of edges (states and actions) in the automaton that are traversed by $\sigma$. Each such edge set describes a policy $\pi$ wherein the player upon reaching a state chooses the action on one of the outgoing edges from the set uniformly at random, or alternatively, if no outgoing edge from the set is present, the player chooses any action uniformly at random. Note that multiple sequences from the training data map can onto the same policy.

 Once the game and the policy are constructed from the training data, we use Algorithm~\ref{alg:info-state-machine-search} to construct an ISM given $T = T(\epsilon)$ and $\lambda$. This is used to construct an MDP, and thus, a policy $\pi_1$ for $\player1$. 
 The policy is tested by using the held out test sequences consisting of the 25\% of the sequences not  used in learning the task model or the policies. Using each sequence as the set of actions chosen by the oblivious $\player2$, we measure the average reward for each episode and the average action prediction score for $\player1$ for various values of $\epsilon, \lambda$.

 Table~\ref{tab:task-action-prediction-performance} shows the size of the ISM, running time of Algo.~\ref{alg:info-state-machine-search} and the performance of the policy for $\player1$ on the held out test sequences. First, we note that the performance in terms of running time and size of the ISM shows trends that are similar to the previous benchmarks reported in Table~\ref{tab:benchmark-performance}. In terms of the held out sequences, we note that our approach is successful in terms of predicting the actions of $\player2$. Given that the cataract data has $41$ actions and Ikea dataset has $32$ actions, our approach performs much better than a random guess. At the same time, the action probability score (the average probability ascribed by the ISM belief's state to $\player2$ action in the current move) is also high given the large space of possible actions. Interestingly, however, we note that changing $\lambda, \epsilon$ has an enormous impact on the running time and size of the ISM but very little impact on the performance on the unseen test sequence. The average probability score shows a  small variations across different values of $\lambda, \epsilon$. We believe that this is a function of the rather small values of $\lambda$ used since it assures us that the ISM tracks the belief state very precisely. 

\section{Conclusion}
We study concurrent stochastic games against oblivious opponents where the opponent (environment) is not necessarily defined as adversarial or cooperative, but rather oblivious that is bounded to choose from a finite set of policies. 
We introduce the notion of \emph{information state machine} (ISM) whose states are mapped to a belief state on the environment policy, and provide the guarantee that the belief states tracked by this automaton stay within a fixed distance of the precise belief state obtained by tracking the entire history for the environment.
%Composition of the ISM with the game results
%an MDP that serves as the starting point for deriving optimal policies.
%We present an experimental evaluation and demonstrate that our synthesized controller can robustly anticipate the policies and actions of the environment
%in order to maximize the reward.
%
In the future, we would like to better characterize the relationship between the various parameters involved in Algorithm~\ref{alg:info-state-machine-search} to provide a tighter condition for its termination. We are also interested in understanding the applicability of these ideas to the more general case of partially observable Markov decision processes. 
% POMDP models. 

% We study policy synthesis in concurrent stochastic games against oblivious opponents. 
% In this framework, the opponent (environment) is not necessarily defined as adversarial or cooperative, but rather oblivious opponent that is bounded to choose from a finite set of policies. 
% We introduce the notion of ``information state machine'' whose states are mapped to a belief state on the environment policy, and provide the guarantee that the belief states tracked by this automaton stay within a fixed distance of the precise belief state obtained by tracking the entire history for the environment.
% Composition of the information state machine with the game results
% an MDP that serves as the starting point for deriving optimal policies.
% We present an experimental evaluation over challenging examples and demonstrate that our synthesized controller can robustly anticipate the policies and actions of the environment in order to maximize the reward.
\begin{credits}
\subsubsection{\ackname} This work was supported in part by the US NSF  under award numbers 1836900 and 1932189.and the NSF IUCRC CAAMS center. All opinions are those of the authors and not necessarily of the NSF.

\subsubsection{\discintname}
The authors have no conflicts of interest to disclose.
\end{credits}

\bibliographystyle{splncs04}
\bibliography{main}
\newpage 
\appendix
\section{Rock Paper Scissors with Memory}
\label{Appendix:rps-mem}
We describe the \textsc{rps-mem} benchmark used in our approach. The state of the game $\scr{G}$ is given as  $S = A_1 \times A_2$ wherein $A_1 = \{r_1, p_1, s_1\}$ and $A_2 = \{r_2, p_2, s_2\}$ while $\Ac1 = A_1$ and $\Ac2 = A_2$. The transition map is given as follows:
\[ P( (s_1, s_2)\ |\ s, a_1, a_2) = \begin{cases} 1 & \text{if}\ s_1 = a_1,\ s_2 = a_2 \\ 
0 & \text{otherwise} \\ \end{cases}
\]
In other words, the state $s$ ``remembers'' the previous action of both players. The reward map for each state is identical to that of the \textsc{rps} game from Example~\ref{ex:rps-example}.

We define $9$ policies for $\player2$.

Policy $\pi_1$ chooses rock/paper with 0.45 probability and scissors with 0.1 probability regardless of the state.
\[ \pi_1(a,b) = \{ r_2: 0.45, p_2: 0.45, s_2: 0.1 \}\,.\]
Likewise, we define policies $\pi_2, \pi_3$.
\[ \pi_2(a,b) = \{ r_2: 0.45, p_2:  0.1, s_2: 0.45\} \]
\[ \pi_3(a,b) = \{ r_2: 0.1, p_2:  0.45, s_2: 0.45\} \]
Policy $\pi_4$: mostly repeat what player $\player1$ played in the previous round.
\[ \pi_4(a,b) = \begin{cases} 
 \{ r_2: 0.8, p_2: 0.1, s_2: 0.1 \} & \text{if}\ a= r_1 \\ 
  \{ r_2: 0.1, p_2: 0.8, s_2: 0.1 \} & \text{if}\ a = p_1 \\ 
   \{ r_2: 0.1, p_2: 0.1, s_2: 0.8 \} & \text{if}\ a = s_1 \\ 
\end{cases}\]

Policy $\pi_5$: mostly play what would have beaten player 1 in the previous round.

\[ \pi_5(a, b) = \begin{cases} 
 \{ r_2: 0.8, p_2: 0.1, s_2: 0.1 \} & \text{if}\ a= s_1 \\ 
  \{ r_2: 0.1, p_2: 0.8, s_2: 0.1 \} & \text{if}\ a = r_1 \\ 
   \{ r_2: 0.1, p_2: 0.1, s_2: 0.8 \} & \text{if}\ a = p_1 \\ 
\end{cases}\]

Policy $\pi_6$: Mostly play what player 1 did not play in the previous round.
\[ \pi_6(a,b) = \begin{cases}
    \{ r_2: 0.1, p_2:0.45, s_2: 0.45\} & \text{if}\ a = r_1 \\ 
     \{ r_2: 0.45, p_2:0.1, s_2: 0.45\} & \text{if}\ a = p_1 \\ 
      \{ r_2: 0.45, p_2:0.45, s_2: 0.1\} & \text{if}\ a = s_1 \\ 
\end{cases}\]

Policy $\pi_7$: mostly repeat what player $\player2$ played in the previous round.
\[ \pi_4(a,b) = \begin{cases} 
 \{ r_2: 0.8, p_2: 0.1, s_2: 0.1 \} & \text{if}\ b= r_2 \\ 
  \{ r_2: 0.1, p_2: 0.8, s_2: 0.1 \} & \text{if}\ b = p_2 \\ 
   \{ r_2: 0.1, p_2: 0.1, s_2: 0.8 \} & \text{if}\ b = s_2 \\ 
\end{cases}\]

Policy $\pi_8$: mostly play what would have beaten $\player2$ in the previous round.

\[ \pi_5(a, b) = \begin{cases} 
 \{ r_2: 0.8, p_2: 0.1, s_2: 0.1 \} & \text{if}\ b= s_2 \\ 
  \{ r_2: 0.1, p_2: 0.8, s_2: 0.1 \} & \text{if}\ b = r_2 \\ 
   \{ r_2: 0.1, p_2: 0.1, s_2: 0.8 \} & \text{if}\ b = p_2 \\ 
\end{cases}\]
Policy $\pi_9$: mostly  play what $\player2$ did not play in the previous round.
\[ \pi_7(a,b) = \begin{cases}
    \{ r_2: 0.1, p_2:0.45, s_2: 0.45\} & \text{if}\ b = r_2 \\ 
     \{ r_2: 0.45, p_2:0.1, s_2: 0.45\} & \text{if}\ b = p_2 \\ 
      \{ r_2: 0.45, p_2:0.45, s_2: 0.1\} & \text{if}\ b = s_2 \\ 
\end{cases}\]

\section{Anticipate and Avoid}
\label{Appendix:ant-avoid}
Anticipate and Avoid game involves a circular arena with $N$ cells labeled $1, \ldots, N$. The state space $S$ encodes joint positions of two players in this arena.

\[ S = \{ (i, j)\ |\ 1 \leq i \leq N, 1 \leq j \leq N \} \,.\]

Let us define $i \oplus 1 $ as the same as $i + 1$ if $1 leq i \leq N-1$ and to be $1$ if $i = 1$. Likewise, we define $i \ominus 1$ as $i-1$ for $2 \leq i \leq N$ and $N$ if $i = 1$.

The actions are $\Ac1= \Ac2 = \{ L, R\}$ standing for left and right, respectively. Let us define $p(j | i, a)$ for a single player as follows:
\[ p(j | i, a) = \begin{cases}
    0.2 & j = i \\ 
    0.8 & j = i \oplus 1, a = R\\
    0.8 & j = i \ominus 1, a = L\\ 
    0 & \text{otherwise}\\
\end{cases}\]
In other words, upon moving left, the player may stay in the same cell with 0.2 probability or move to "previous" cell with 0.8 probability and similarly for moving right. 

The reward map is defined by first defining a state distance function:
\[ \rho(i, j) = \begin{cases}
\min(j-i, i-j +N) & \text{if}\ i \leq j \\ 
\min(i-j, j-i+N) & \text{if}\ i > j\\
\end{cases} \]

We define the reward for state/actions as
\[ R((i,j), a_1, a_2) = \begin{cases}
    -10 & i = j \\ 
    -5 & i \not= j \ \land\ \rho(i,j) \leq N/10 \\ 
    0 & i \not= j \ \land\ \rho(i,j) \in (N/10, 3N/10]\\
    1 & \text{otherwise}\\
\end{cases}\]
In other words, the reward structure incentivizes $i,j$ positions to be farther apart than $3N/10$.

Player $2$ can play one of four policies of the form  $\mathsf{target}_t$ for $t = 1, \lceil N/4 \rceil, \lceil 2N/4 \rceil, \lceil 3N/4\rceil$, where the policy $\mathsf{target}(j)$ is defined as 

\[ \mathsf{target}_t(i,j) = \begin{cases}
    \{ L: 0.8, R: 0.2 \} &  j > t \ \land\ (t - j+ N \geq j -t )  \\ 
     \{ L: 0.8, R: 0.2 \} &  j < t \ \land\ (j -t +N \leq t -j )  \\ 
    \{ L: 0.2, R: 0.8 \} &  j > t \ \land\ (t-j +N \leq t -j ) \\
    \{ L: 0.2, R: 0.8 \} & j < t \ \land\ (j -t +N \geq t -j )  \\
    \{ L:0.5, R: 0.5 \} & \text{otherwise}\\
\end{cases}\]

\section{Appendix: Ikea Furniture Assembly and Cataract Surgery Graphs}\label{app:ikea-cataract-benchmark-details}

The ikea furniture assembly dataset was taken from the previous work of Ben-Shabat et al~\cite{Ben-Shabat+Others/2021/WACV}. It involves sequences of tasks for four different furniture types with roughly $90$ sequences for each furniture type. We employed a  $75\%$-$25\%$ training/testing data split. The tool flexfringe was used to learn an automaton model using sequences in the training data. 

\begin{table}
\begin{tabular}{|cc|cc|cc|}
\hline 
0 & flip table top &1 & pick up leg &2 & align leg screw with table thread \\
3 & spin leg &4 & other &5 & tighten leg \\
6 & rotate table &7 & flip table &8 & pick up shelf \\
9 & attach shelf to table &10 & pick up table top &11 & lay down table top \\
12 & push table &13 & flip shelf &14 & lay down leg \\
15 & lay down shelf &16 & push table top &17 & pick up side panel \\
18 & align side panel holes with front panel dowels &19 & attach drawer side panel &20 & pick up bottom panel \\
21 & slide bottom of drawer &22 & pick up back panel &23 & attach drawer back panel \\
24 & pick up pin &25 & insert drawer pin &26 & position the drawer right side up \\
27 & pick up front panel &28 & lay down bottom panel &29 & lay down front panel \\
30 & lay down back panel &31 & lay down side panel & & \\
\hline 
\end{tabular}
\caption{Action IDs and their description for the IKEA furniture assembly benchmark.}
\end{table}

\begin{table}
\begin{tabular}{|cc|cc|cc|}
\hline 
0 & +Bonn forceps  & 1 & +secondary incision knife  & 2 & -Bonn forceps \\
3 & -secondary incision knife  & 4 & +primary incision knife  & 5 & -primary incision knife \\
6 & +viscoelastic cannula  & 7 & -viscoelastic cannula  & 8 & +capsulorhexis cystotome \\
9 & -capsulorhexis cystotome  & 10 & +capsulorhexis forceps  & 11 & -capsulorhexis forceps \\
12 & +hydrodissection canula  & 13 & -hydrodissection canula  & 14 & +phacoemulsifier handpiece \\
15 & +micromanipulator  & 16 & -phacoemulsifier handpiece  & 17 & -micromanipulator \\
18 & +irrigation/aspiration handpiece  & 19 & -irrigation/aspiration handpiece  & 20 & +implant injector \\
21 & -implant injector  & 22 & +Rycroft canula  & 23 & -Rycroft canula \\
24 & +Troutman forceps  & 25 & -Troutman forceps  & 26 & +cotton \\
27 & -cotton  & 28 & +Charleux canula  & 29 & -Charleux canula \\
30 & +suture needle  & 31 & -suture needle  & 32 & +Vannas scissors \\
33 & -Vannas scissors  & 34 & +needle holder  & 35 & -needle holder \\
36 & +vitrectomy handpiece  & 37 & -vitrectomy handpiece  & 38 & +biomarker \\
39 & -biomarker  & 40 & +Mendez ring  & 41 & -Mendez ring \\
\hline 
\end{tabular}
\caption{Action IDs and their description for the cataract surgery benchmark. A ``+'' sign before a tool indicates its introduction during a particular step, whereas a ``-'' sign indicates its removal.}
\end{table}

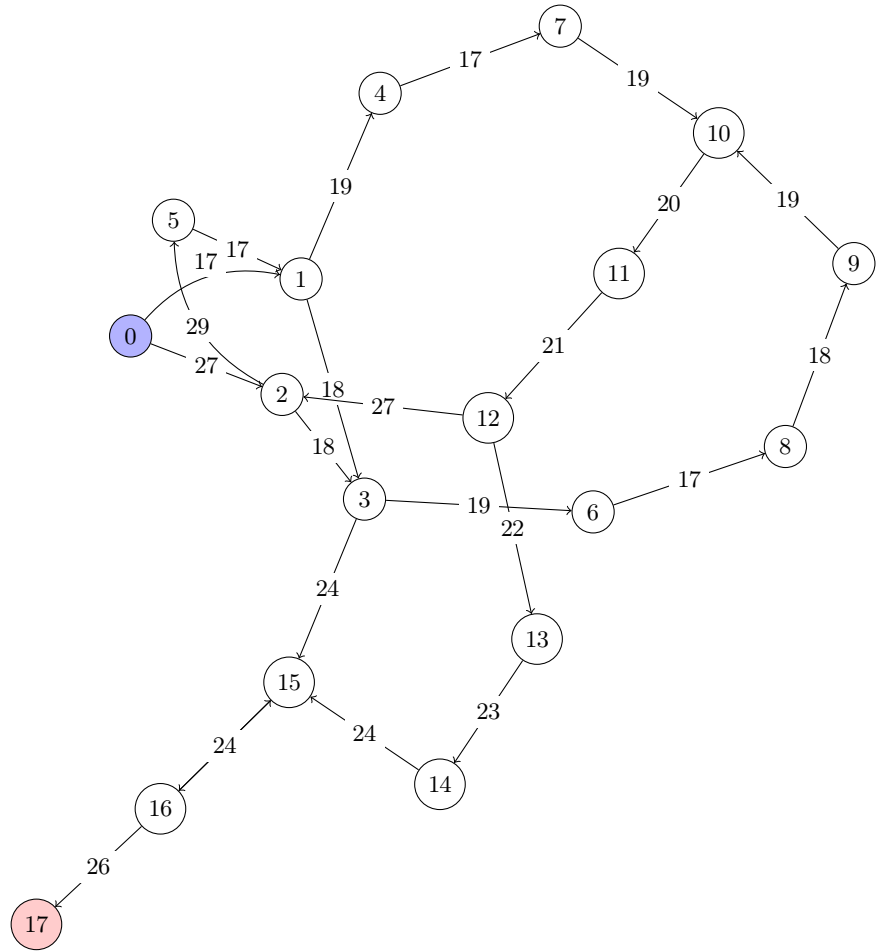
\begin{figure}[t]
\begin{tikzpicture}[x=0.3mm, y=0.3mm, every node/.style={fill=white}, my_node/.style={circle, draw=black}]
      \draw
        (68.733, 278.95) node[my_node, circle, fill=blue!30] (0){0}
        (144.28, 304.24) node[my_node] (1){1}
        (135.86, 253.03) node[my_node] (2){2}
        (172.39, 206.6) node[my_node] (3){3}
        (179.31, 386.57) node[my_node] (4){4}
        (87.729, 330.29) node[my_node] (5){5}
        (273.69, 200.93) node[my_node] (6){6}
        (138.97, 125.35) node[my_node] (15){15}
        (259.14, 416.35) node[my_node] (7){7}
        (358.91, 229.94) node[my_node] (8){8}
        (329.37, 368.98) node[my_node] (10){10}
        (389.25, 311.05) node[my_node] (9){9}
        (285.21, 306.65) node[my_node] (11){11}
        (227.22, 242.63) node[my_node] (12){12}
        (248.85, 144.54) node[my_node] (13){13}
        (205.82, 80.128) node[my_node] (14){14}
        (81.957, 69.261) node[my_node] (16){16}
        (27.0, 18.0) node[my_node, fill=red!20] (17){17};
      \begin{scope}[->]
        \draw[bend left] (0) to node[above] {17} (1);
        \draw (0) to node[] {27} (2);
        \draw (1) to node[] {18} (3);
        \draw (1) to node[] {19} (4);
        \draw (2) to node[] {18} (3);
        \draw[bend left] (2) to node[] {29} (5);
        \draw (3) to node[] {19} (6);
        \draw (3) to node[] {24} (15);
        \draw (4) to node[] {17} (7);
        \draw (5) to node[] {17} (1);
        \draw (6) to node[] {17} (8);
        \draw (15) to node[] {25} (16);
        \draw (7) to node[] {19} (10);
        \draw (8) to node[] {18} (9);
        \draw (10) to node[] {20} (11);
        \draw (9) to node[] {19} (10);
        \draw (11) to node[] {21} (12);
        \draw (12) to node[] {22} (13);
        \draw (12) to node[] {27} (2);
        \draw (13) to node[] {23} (14);
        \draw (14) to node[] {24} (15);
        \draw (16) to node[] {24} (15);
        \draw (16) to node[] {26} (17);
      \end{scope}
    \end{tikzpicture}
\caption{IKEA Shelf Drawer Assembly Task Machine. }\label{fig:ikea-shelf-drawer-assembly}
\end{figure}

\begin{figure}[t]
\begin{tikzpicture}[x=0.3mm, y=0.3mm, every node/.style={fill=white}, my_node/.style={circle, draw=black}]
      \draw
        (70.464, 19.701) node[my_node, fill=blue!30] (0){0}
        (27.0, 60.115) node[my_node] (1){1}
        (100.01, 94.904) node[my_node] (2){2}
        (48.349, 154.16) node[my_node] (3){3}
        (179.55, 92.364) node[my_node] (11){11}
        (114.1, 200.71) node[my_node] (4){4}
        (203.52, 234.1) node[my_node] (5){5}
        (153.15, 155.28) node[my_node] (6){6}
        (278.44, 187.69) node[my_node] (7){7}
        (223.6, 292.08) node[my_node] (14){14}
        (143.86, 256.44) node[my_node] (15){15}
        (235.32, 119.56) node[my_node] (8){8}
        (284.03, 46.318) node[my_node] (9){9}
        (217.51, 18.0) node[my_node] (10){10}
        (270.02, 145.85) node[my_node] (12){12}
        (294.02, 243.52) node[my_node] (13){13}
        (95.724, 324.76) node[my_node] (16){16}
        (60.302, 392.2) node[my_node, fill=red!20] (17){17};
      \begin{scope}[->]
        \draw (0) to node[] {0} (1);
        \draw (0) to node[] {1} (2);
        \draw (1) to node[] {1} (2);
        \draw (2) to node[] {2} (3);
        \draw (2) to node[] {3} (11);
        \draw (3) to node[] {3} (4);
        \draw[bend left] (11) to node[] {1} (12);
        \draw (11) to node[] {5} (6);
        \draw (4) to node[] {1} (5);
        \draw (4) to node[] {5} (6);
        \draw (5) to node[] {2} (7);
        \draw (5) to node[] {3} (14);
        \draw (6) to node[] {1} (2);
        \draw (6) to node[] {7} (15);
        \draw (7) to node[] {3} (8);
        \draw (14) to node[] {7} (15);
        \draw (15) to node[] {8} (16);
        \draw (8) to node[] {1} (9);
        \draw (8) to node[] {5} (6);
        \draw (9) to node[] {2} (10);
        \draw (10) to node[] {3} (11);
        \draw[bend right] (12) to node[] {2} (13);
        \draw (13) to node[] {3} (14);
        \draw (16) to node[] {9} (17);
      \end{scope}
    \end{tikzpicture}
\caption{IKEA TV Bench Assembly Task Machine}\label{fig:ikea-tv-bench-assembly-task}
\end{figure}
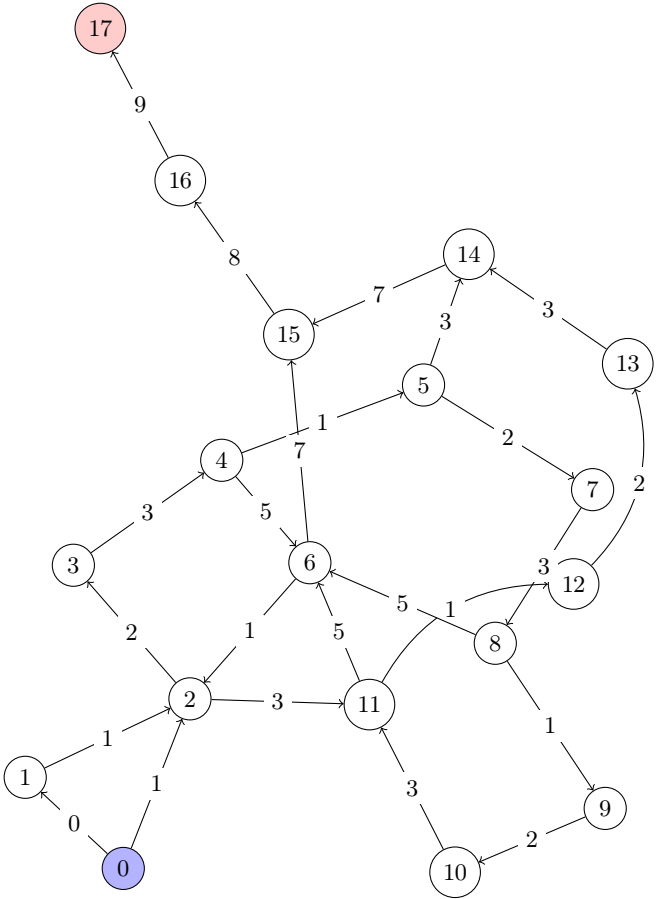

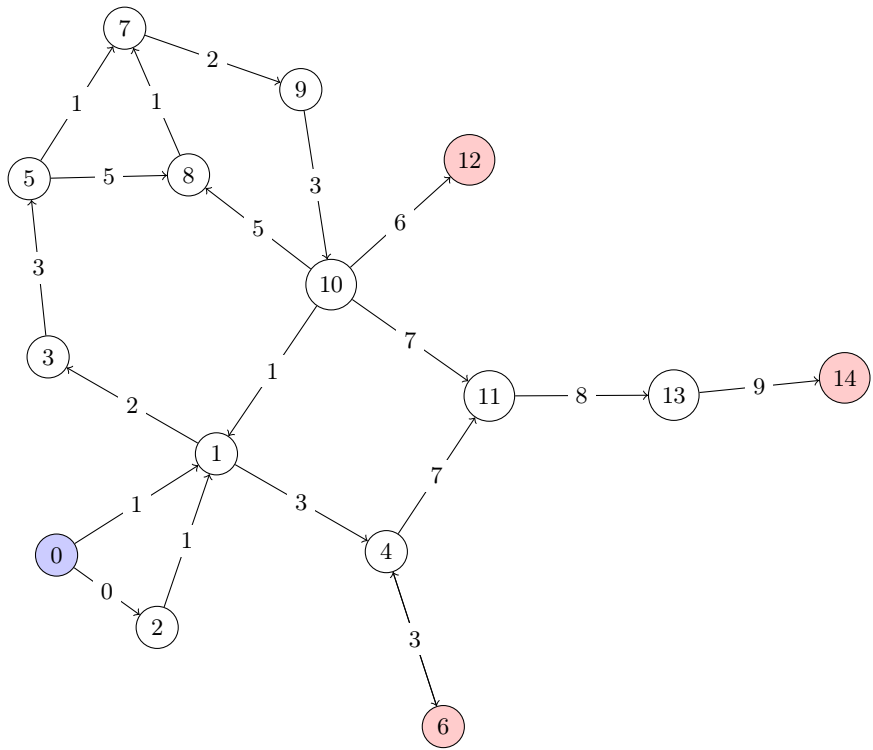
\begin{figure}
\begin{tikzpicture}[x=0.3mm, y=0.3mm, every node/.style={fill=white}, my_node/.style={circle, draw=black}]
      \draw
        (39.124, 94.062) node[my_node, fill=blue!20] (0){0}
        (109.98, 138.96) node[my_node] (1){1}
        (83.714, 61.992) node[my_node] (2){2}
        (35.383, 181.95) node[my_node] (3){3}
        (185.26, 95.587) node[my_node] (4){4}
        (27.0, 260.98) node[my_node] (5){5}
        (210.51, 18.0) node[my_node, fill=red!20] (6){6}
        (230.91, 164.62) node[my_node] (11){11}
        (69.336, 327.73) node[my_node] (7){7}
        (97.597, 262.67) node[my_node] (8){8}
        (147.36, 300.49) node[my_node] (9){9}
        (160.82, 214.0) node[my_node] (10){10}
        (222.12, 269.34) node[my_node, fill=red!20] (12){12}
        (312.64, 165.0) node[my_node] (13){13}
        (388.39, 172.65) node[my_node, fill=red!20] (14){14};
      \begin{scope}[->]
        \draw (0) to node[] {1} (1);
        \draw (0) to node[] {0} (2);
        \draw (1) to node[] {2} (3);
        \draw (1) to node[] {3} (4);
        \draw (2) to node[] {1} (1);
        \draw (3) to node[] {3} (5);
        \draw (4) to node[] {1} (6);
        \draw (4) to node[] {7} (11);
        \draw (5) to node[] {1} (7);
        \draw (5) to node[] {5} (8);
        \draw (6) to node[] {3} (4);
        \draw (11) to node[] {8} (13);
        \draw (7) to node[] {2} (9);
        \draw (8) to node[] {1} (7);
        \draw (9) to node[] {3} (10);
        \draw (10) to node[] {1} (1);
        \draw (10) to node[] {7} (11);
        \draw (10) to node[] {5} (8);
        \draw (10) to node[] {6} (12);
        \draw (13) to node[] {9} (14);
      \end{scope}
    \end{tikzpicture}
\caption{IKEA Coffee Table Assembly Task Machine}\label{fig:ikea-coffee-table-assembly-task-machine}
\end{figure}

\begin{figure}
\begin{tikzpicture}[x=0.2mm, y=0.3mm, every node/.style={fill=white}, my_node/.style={circle, draw=black}]
      \draw
        (205.1, 49.248) node[my_node, fill=blue!20] (0){0}
        (128.87, 18.0) node[my_node] (1){1}
        (228.97, 126.86) node[my_node] (2){2}
        (86.523, 88.448) node[my_node] (3){3}
        (189.61, 201.52) node[my_node] (5){5}
        (114.34, 164.17) node[my_node] (4){4}
        (206.22, 281.9) node[my_node] (6){6}
        (216.99, 373.9) node[my_node] (7){7}
        (273.63, 458.77) node[my_node] (8){8}
        (170, 320.78) node[my_node] (9){9}
        (322.19, 440.67) node[my_node] (11){11}
        (327.76, 500.99) node[my_node] (10){10}
        (218.22, 434.23) node[my_node] (12){12}
        (413.07, 447.28) node[my_node] (27){27}
        (128.72, 462.15) node[my_node] (13){13}
        (54.829, 507.07) node[my_node] (14){14}
        (27.0, 584.82) node[my_node] (15){15}
        (71.003, 652.9) node[my_node] (16){16}
        (42.3, 728.81) node[my_node] (17){17}
        (154.95, 661.35) node[my_node] (18){18}
        (122.71, 738.04) node[my_node] (19){19}
        (222.97, 610.05) node[my_node] (20){20}
        (291.43, 553.71) node[my_node] (21){21}
        (344.66, 622.87) node[my_node] (22){22}
        (370.87, 544.71) node[my_node] (25){25}
        (420.44, 655.35) node[my_node] (23){23}
        (416.82, 593.13) node[my_node] (24){24}
        (444.95, 525.64) node[my_node] (26){26}
        (530.45, 527.64) node[my_node] (28){28}
        (497.77, 445.75) node[my_node] (29){29}
        (571.5, 606.82) node[my_node] (30){30}
        (607.65, 500.28) node[my_node] (31){31}
        (559.79, 689.37) node[my_node] (32){32}
        (683.6, 487.33) node[my_node, fill=red!20] (33){33}
        (483.01, 727.13) node[my_node] (34){34}
        (396.01, 706.48) node[my_node] (35){35};
      \begin{scope}[->]
        \draw (0) to node[] {0} (1);
        \draw (0) to node[] {1} (2);
        \draw (1) to node[] {1} (3);
        \draw (2) to node[] {3} (5);
        \draw (3) to node[] {2} (4);
        \draw (5) to node[] {1} (6);
        \draw (4) to node[] {3} (5);
        \draw (6) to node[] {3} (7);
        \draw (7) to node[] {0} (8);
        \draw (7) to node[] {4} (9);
        \draw (8) to node[] {2} (11);
        \draw (8) to node[] {4} (10);
        \draw (9) to node[] {5} (12);
        \draw (11) to node[] {5} (12);
        \draw (11) to node[] {15} (27);
        \draw (10) to node[] {2} (11);
        \draw (12) to node[] {6} (13);
        \draw (27) to node[] {17} (29);
        \draw (13) to node[] {7} (14);
        \draw (14) to node[] {8} (15);
        \draw (15) to node[] {9} (16);
        \draw (16) to node[] {10} (17);
        \draw (16) to node[] {12} (18);
        \draw (17) to node[] {11} (19);
        \draw (18) to node[] {13} (20);
        \draw (19) to node[] {12} (18);
        \draw (20) to node[] {14} (21);
        \draw (21) to node[] {0} (8);
        \draw (21) to node[] {15} (22);
        \draw (21) to node[] {16} (25);
        \draw (22) to node[] {16} (23);
        \draw (22) to node[] {17} (24);
        \draw (25) to node[] {14} (21);
        \draw (25) to node[] {18} (26);
        \draw (23) to node[] {17} (24);
        \draw (24) to node[] {16} (25);
        \draw (24) to node[] {18} (26);
        \draw (26) to node[] {15} (27);
        \draw[bend left] (26) to node[] {19} (28);
        \draw (28) to node[] {6} (30);
        \draw[bend left] (28) to node[] {18} (26);
        \draw (28) to node[] {22} (31);
        \draw (29) to node[] {19} (28);
        \draw (30) to node[] {7} (32);
        \draw (31) to node[] {23} (33);
        \draw (32) to node[] {20} (34);
        \draw (33) to node[] {22} (31);
        \draw (34) to node[] {21} (35);
        \draw (35) to node[] {15} (22);
      \end{scope}
    \end{tikzpicture}
\caption{Cataract Surgery Task Machine}\label{fig:cataract-surgery-task-machine}
\end{figure}
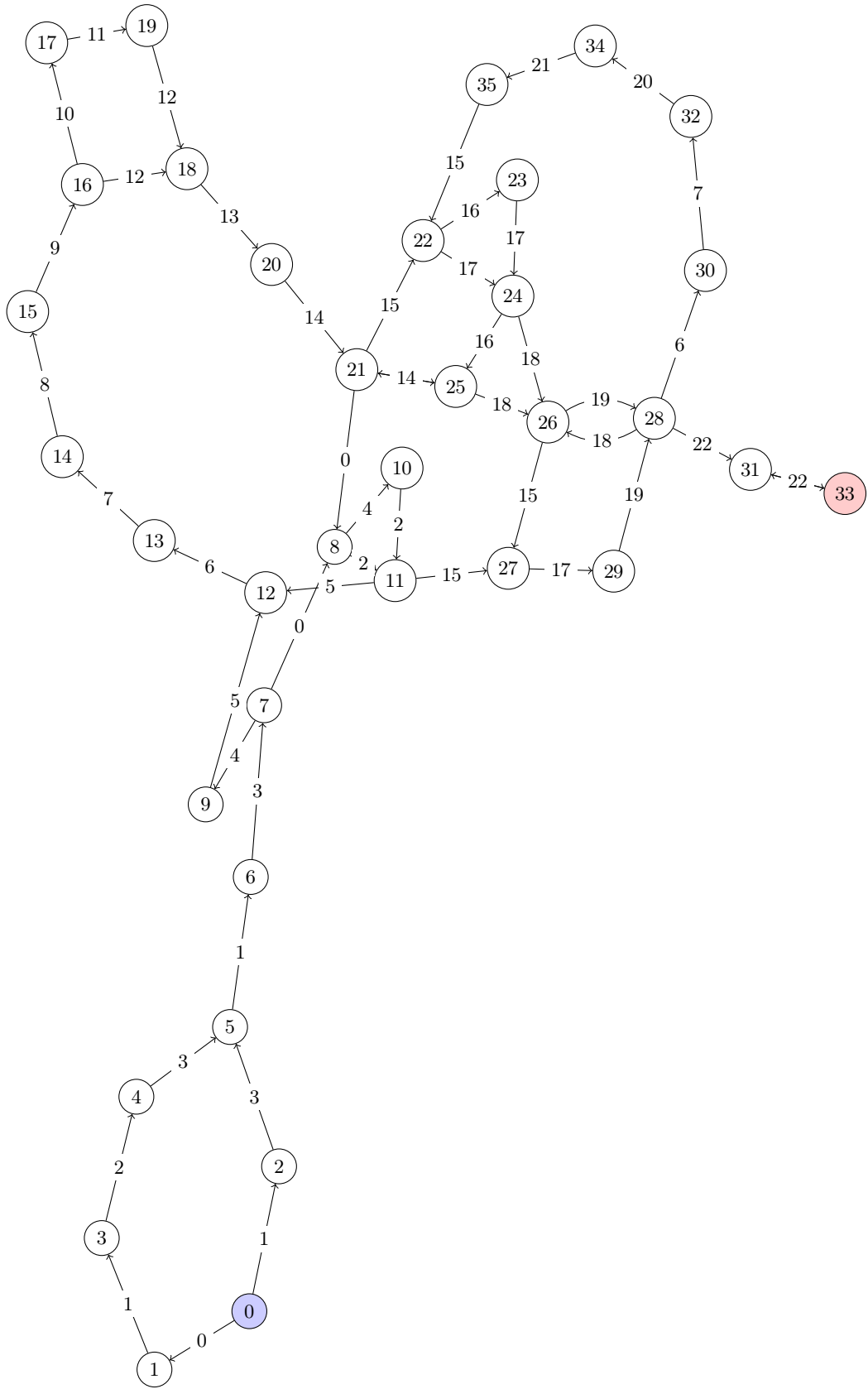
\end{document}